\documentclass{article}

\PassOptionsToPackage{numbers, sort&compress}{natbib}
\usepackage[final]{neurips_2021}

\usepackage[utf8]{inputenc} 
\usepackage[T1]{fontenc}    
\usepackage{hyperref}       
\usepackage{url}            
\usepackage{booktabs}       
\usepackage{amsfonts}       
\usepackage{nicefrac}       
\usepackage{microtype}      
\usepackage{xspace}
\usepackage{caption}
\usepackage{subcaption}
\usepackage[inline]{enumitem}
\usepackage{wrapfig}

\usepackage{amsmath}
\usepackage{amsfonts}
\usepackage{amssymb}
\usepackage{amsthm}
\usepackage{bm}
\usepackage{bbm}
\usepackage{mathtools}
\usepackage{enumitem}
\usepackage{thmtools,thm-restate}
\usepackage{algorithm}
\usepackage{algorithmic}
\usepackage{natbib}
\usepackage[capitalise]{cleveref}
\usepackage{color}
\usepackage[dvipsnames]{xcolor}
\usepackage{graphicx}
\usepackage{comment}

\newcommand{\rev}[1]{{\leavevmode\color{black}#1}}

\theoremstyle{plain}
\newtheorem{lemma}{Lemma}[section]

\theoremstyle{definition}
\newtheorem{definition}{Definition}[section]

\theoremstyle{remark}

\def\AA{\mathcal{A}}\def\BB{\mathcal{B}}\def\CC{\mathcal{C}}
\def\DD{\mathcal{D}}

\def\LL{\mathcal{L}}
\def\MM{\mathcal{M}}

\def\SS{\mathcal{S}}

\def\Ebb{\mathbb{E}}

\def\Rbb{\mathbb{R}}

\def\R{\Rbb}
\def\one{{\mathbbm1}}

\def\*{\star}
\DeclareMathSymbol{\mhef}{\mathord}{operators}{`\-}

\newcommand{\norm}[1]{ \| #1 \|  }

\DeclareMathOperator*{\argmax}{arg\,max}

\newcommand{\E}{\Ebb}

\def\h{h}
\def\Const{\frac{(1-\lambda)\gamma }{1-\gamma\lambda}}

\def\algo{HuRL\xspace}
\def\algofull{Heuristic-Guided Reinforcement Learning\xspace}

\title{Heuristic-Guided Reinforcement Learning}

\author{%
  Ching-An Cheng \\
  Microsoft Research\\
  Redmond, WA\\
  \texttt{chinganc@microsoft.com} \\
  \And
  Andrey Kolobov \\
  Microsoft Research\\
  Redmond, WA\\
  \texttt{akolobov@microsoft.com}\\
   \And
   Adith Swaminathan \\
   Microsoft Research \\
   Redmond, WA\\
   \texttt{adswamin@microsoft.com} 
}

\begin{document}

\maketitle

\begin{abstract}
  We provide a framework for accelerating reinforcement learning (RL) algorithms by heuristics constructed from domain knowledge or offline data.
  Tabula rasa RL algorithms require environment interactions or computation that scales with the horizon of the sequential decision-making task.
  Using our framework, we show how heuristic-guided RL induces a much shorter-horizon subproblem that provably solves the original task. 
  Our framework can be viewed as a horizon-based regularization for controlling bias and variance in RL under a finite interaction budget.
  On the theoretical side, we characterize properties of a good heuristic and its impact on RL acceleration. In particular, we introduce the novel concept of an {improvable heuristic}, a heuristic that allows an RL agent to extrapolate beyond its prior knowledge.
  On the empirical side, we instantiate our framework to accelerate several state-of-the-art algorithms in simulated robotic control tasks and procedurally generated games.
  Our framework complements the rich literature on warm-starting RL with expert demonstrations or exploratory datasets, and introduces a principled method for injecting prior knowledge into RL.
\end{abstract}

\section{Introduction} 
\label{sec:intro}

Many recent empirical successes of reinforcement learning (RL) require solving problems with very long decision-making horizons.
OpenAI Five~\citep{berner2019dota} used episodes that were $20000$ timesteps on average, while AlphaStar~\citep{vinyals2019grandmaster} used roughly $5000$ timesteps. 
Long-term credit assignment is a very challenging statistical problem, with the sample complexity 
growing quadratically (or worse) with the horizon~\citep{dann2015sample}. Long horizons (or, equivalently, large discount factors) also increase RL's computational burden, leading to slow optimization convergence~\citep{sidford2018near}. 
This makes RL algorithms require prohibitively large amounts of interactions and compute: even with tuned hyperparameters, AlphaStar needed over $10^8$ samples and OpenAI Five needed over $10^7$ PFLOPS of compute. 

A popular approach to mitigate the statistical and computational issues of tabula rasa RL methods is to warm-start or regularize learning with prior knowledge~\citep{tessler2020maximizing,NEURIPS2019_eba237ec,bhardwaj2020blending,farahmand2016truncated,vinyals2019grandmaster,berner2019dota,nair2020accelerating,hester2018deep}. 
For instance, AlphaStar learned a policy and value function from human demonstrations and regularized the RL agent using imitation learning (IL).
AWAC~\citep{nair2020accelerating} warm-started a policy using batch policy optimization on exploratory datasets. 
While these approaches have been effective in different domains, none of them explicitly address RL's complexity dependence on horizon.

In this paper, we propose a complementary regularization 
technique that relies on  heuristic value functions, or \emph{heuristics}\footnote{\rev{We borrow this terminology from the planning literature to refer to guesses of $V^*$ in an MDP~\citep{kolobov2012}. 
}} for short, to effectively shorten the problem horizon faced by an online RL agent \rev{for fast learning}. 
We call this approach \algofull (\algo). The core idea is simple: given a Markov decision process (MDP) $\MM=(\SS, \AA, P, r, \gamma)$ and a heuristic $\h:\SS\to\R$, we select a mixing coefficient $\lambda\in[0,1]$ and have the agent solve a new MDP $\widetilde{\MM}=(\SS, \AA, P, \widetilde{r}, \widetilde{\gamma})$ with a reshaped reward and a smaller discount \rev{(i.e. a shorter horizon)}:
\begin{align} \label{eq:short summary}
    \widetilde{r}(s,a) \coloneqq r(s,a) + (1-\lambda)\gamma \E_{s'\sim P(\cdot| s,a)}[\h(s')]
    \quad
    \text{and}
    \quad
    \widetilde{\gamma} \coloneqq \lambda\gamma.
\end{align}

%
\algo effectively introduces horizon-based regularization that determines whether long-term value information should come from collected experiences or the heuristic. 
By modulating the effective horizon via $\lambda$, we trade off the bias and the complexity of solving the reshaped MDP. \algo with $\lambda=1$ recovers the original problem and with $\lambda=0$ creates an easier contextual bandit problem~\citep{pmlr-v119-foster20a}.


A heuristic $\h$ in \algo represents a prior guess of the desired long-term return of states, which ideally is the optimal value function $V^*$ of the unknown MDP $\MM$. 
\rev{When the heuristic $\h$ captures the state ordering of $V^*$ well, conceptually, it becomes possible to make good long-term decisions by short-horizon planning or even acting greedily.}
How do we construct a good heuristic? In the planning literature, this is typically achieved by solving a relaxation of the original problem \cite{Hoffmann_2001,Richter_2010,akolobov-aaai10}. 
Alternatively, one can learn it from batch data collected by exploratory behavioral policies (as in offline RL~\citep{gulcehre2021regularized}) or from expert policies (as in IL~\citep{cheng2020policy}).\footnote{We consider the RL setting for imitation where we suppose the rewards of expert trajectories are available.}
\rev{For some dense reward problems, 
a zero heuristic can be effective in reducing RL complexity, as exploited by the guidance discount framework~\citep{blackwell1962discrete,petrik2008biasing,jiang2015dependence,jiang2016structural,chen2018improving,amit2020discount}.}
In this paper, we view heuristics as a unified representation of various forms of prior knowledge, such as expert demonstrations, exploratory datasets, and engineered guidance. 

Although the use of heuristics to accelerate search has been popular in planning and control algorithms, e.g., A*~\citep{hart1968formal}, MCTS~\citep{browne2012survey}, and MPC~\citep{zhong2013value,bhardwaj2020blending,hoeller2020deep,bejjani2018planning}, its theory is less developed for settings where the MDP is \emph{unknown}.
\rev{The closest work in RL is  
potential-based reward shaping (PBRS)~\citep{ng1999policy}, which reshapes the reward into  $\bar{r}(s,a) = r(s,a)+\gamma \E_{s'|s,a}[\h(s')]-\h(s)$ while keeping the original discount.
PBRS can use any heuristic to reshape the reward while preserving the ordering of policies. 
However, giving PBRS rewards to an RL algorithm does not necessarily lead to faster learning, 
because the base RL algorithm 
would still seek to explore to resolve long-term credit assignment.
\algo~allows common RL algorithms to leverage the short-horizon potential provided by a heuristic to learn faster.
}

In this work, we provide a theoretical foundation of \algo to enable adopting heuristics and horizon reduction for accelerating RL, \rev{combining advances from the PBRS and the guidance discount literatures.}
On the theoretical side, we derive a bias-variance decomposition of \algo's horizon-based regularization in order to characterize the solution quality as a function of $\lambda$ and $\h$. 
Using this insight, we provide sufficient conditions for achieving an effective trade-off, including properties required of a base RL algorithm that solves the reshaped MDP $\widetilde{\MM}_\lambda$. 
Furthermore, we define the novel concept of an \emph{improvable} heuristic and prove that good heuristics for \algo can be constructed from data using existing \emph{pessimistic} offline RL algorithms (such as pessimistic value iteration~\citep{jin2020pessimism,liu2020provably}).

%


The effectiveness of \algo depends on the heuristic quality, so we design \algo to employ a sequence of mixing coefficients (i.e. $\lambda$s) that increases as the agent gathers more data from the environment.
Such a strategy induces a learning curriculum that enables \algo to remain robust to non-ideal heuristics.  
\algo starts off by guiding the agent's search direction with a heuristic. 
As the agent becomes more experienced, it gradually removes the guidance and lets the agent directly optimize the true long-term return.
%
We empirically validate \algo in MuJoCo~\citep{todorov2012mujoco} robotics control problems and Procgen games~\citep{procgen} 
with various heuristics and base RL algorithms.
The experimental results demonstrate the versatility and effectiveness of \algo in accelerating RL algorithms.

%
%


\vspace{-1mm}
\section{Preliminaries}
\vspace{-1mm}

\subsection{Notation}
\vspace{-1mm}

We focus on discounted infinite-horizon Markov Decision Processes (MDPs) for ease of exposition. The technique proposed here can be extended to other MDP settings.\footnote{The results here can be readily applied to finite-horizon MDPs; for other infinite-horizon MDPs, we need further, e.g., mixing assumptions for limits to exist.}
A discounted infinite-horizon MDP is denoted as a 5-tuple  $\MM=(\SS, \AA, P, r, \gamma)$,
where $\SS$ is the state space, $\AA$ is the action space, $P(s'|s,a)$ is the transition dynamics, $r(s,a)$ is the reward function, and $\gamma \in [0,1)$ is the discount factor.
Without loss of generality, we assume $r:\SS\times\AA\to[0,1]$. We allow the state and action spaces $\SS$ and $\AA$ to be either discrete or continuous.
Let $\Delta(\cdot)$ denote the space of probability distributions. A decision-making policy $\pi$ is a conditional distribution $\pi:\SS\to\Delta(\AA)$, which can be deterministic.
We define some shorthand for writing expectations:
For a state distribution $d\in\Delta(\SS)$ and a function ${V : \SS\to\R}$, we define $V(d) \coloneqq \E_{s\sim d}[V(s)]$; similarly, for a policy $\pi$ and a function $Q :\SS\times\AA\to\R$, we define ${Q(s, \pi) \coloneqq \E_{a\sim\pi(\cdot|s)}[Q(s,a)]}$. Lastly, we define $\E_{s'|s,a} \coloneqq \E_{s'\sim P(\cdot|s,a)}$.

Central to solving MDPs are the concepts of value functions and average distributions.
For a policy $\pi$, we define its state value function $V^\pi$ as
$
    V^\pi(s) \coloneqq \E_{\rho_s^\pi} \left[ \sum_{t=0}^\infty \gamma^t r(s_t, a_t)  \right],
$
where $\rho_s^\pi$ denotes the trajectory distribution of $s_0, a_0, s_1, \dots$
induced by running $\pi$ starting from $s_0 =s$.
We define the state-action value function (or the Q-function) as
$Q^\pi(s,a) \coloneqq r(s,a) + \gamma \E_{s'|s,a} [V^\pi(s')]$.
We denote the optimal policy as $\pi^*$ and its state value function as $V^* \coloneqq V^{\pi^*}$. Under the assumption that rewards are in $[0,1]$, we have $V^\pi(s), Q^\pi(s,a)\in[0, \frac{1}{1-\gamma}]$ for all $\pi$, $s\in\SS$, and $a\in\AA$.
We denote the initial state distribution of interest as $d_0 \in \Delta(\SS)$ and the state distribution of policy $\pi$ at time $t$ as $d_t^\pi$, with $d_0^\pi = d_0$.
Given $d_0$, we define the average state distribution of a policy $\pi$ as
$
    d^\pi \coloneqq (1-\gamma) \sum_{t=0}^\infty \gamma^t  d_t^\pi
$.
With a slight abuse of notation, we also write $d^\pi(s,a) \coloneqq d^\pi(s)\pi(a|s)$.
%

\vspace{-1mm}
\subsection{Setup: Reinforcement Learning with Heuristics} \vspace{-1mm} 
\label{sec:RL with heuristic setup}

We consider RL with prior knowledge expressed in the form of a \rev{heuristic value function}.
The goal is to find a policy $\pi$ that has high return through interactions with an unknown MDP $\MM$, i.e., 
    $\max_\pi V^\pi(d_0)$.
While the agent here does not fully know 
$\MM$, we suppose that, before interactions start the agent is provided with a heuristic $\h:\SS\to\R$ which the agent can query throughout learning.

The heuristic $\h$ represents a prior guess of the optimal value function $V^*$ of $\MM$. 
Common sources of heuristics are domain knowledge as typically employed in planning, and logged data collected by exploratory or by expert behavioral policies.
In the latter, a heuristic guess of $V^*$ can be computed from the data by offline RL algorithms. For instance, when we have trajectories of an expert behavioral policy, 
Monte-Carlo regression estimate of the observed returns 
may be a good guess of $V^*$.


Using heuristics to solve MDP problems has been popular in planning and control, but its usage is rather limited in RL.  The closest provable technique in RL is PBRS~\citep{ng1999policy}, where the reward is modified into 
    $\overline{r}(s,a) \coloneqq r(s,a) + \gamma \E_{s'|s,a}[\h(s')] - \h(s)$.
It can be shown that this transformation  does not introduce bias into the policy ordering, and therefore solving the new MDP $\overline{\MM} \coloneqq (\SS,\AA,P,\overline{r},\gamma)$ would yield the same optimal policy $\pi^*$ of $\MM$.
%
%

Conceptually when the heuristic is the optimal value function $\h = V^*$, the agent should be able to find the optimal policy $\pi^*$ of $\MM$ by acting myopically, as $V^*$ already contains all necessary long-term information for good decision making.
\rev{However, running an RL algorithm with the PBRS reward (i.e. solving $\overline{\MM} \coloneqq (\SS,\AA,P,\overline{r},\gamma)$) does not take advantage of this shortcut. To make learning efficient, we need to also let the base RL algorithm know that acting greedily (i.e., using a smaller discount) with the shaped reward can yield good policies. 
An intuitive idea 
is to run the RL algorithm to maximize $\overline{V}^\pi_\lambda(d_0)$, where $\overline{V}^\pi_\lambda$ denotes the value function of $\pi$ in an MDP $\overline{\MM}_\lambda \coloneqq (\SS,\AA,P,\overline{r},\lambda\gamma)$ for some $\lambda\in[0,1]$.
%
However this does not always work. For example, when $\lambda=0$, $\max_\pi \overline{V}^\pi_\lambda(d_0)$ only optimizes for the initial states $d_0$, but obviously the agent is going to encounter other states in $\MM$. 
We next propose a provably correct version, \algo, to leverage this short-horizon insight. 
}

\vspace{-1mm}
\section{Heuristic-Guided Reinforcement Learning}
\vspace{-1mm}

We propose a general framework, \algo, for leveraging heuristics to accelerate RL. 
In contrast to tabula rasa RL algorithms that attempt to directly solve the long-horizon MDP $\MM$, \algo uses a heuristic to guide the agent in solving a sequence of short-horizon MDPs so as to amortize the complexity of long-term credit assignment.
In effect, \algo creates a heuristic-based learning curriculum to help the agent learn faster. 

\vspace{-1mm}
\subsection{Algorithm}
\vspace{-1mm}

\algo takes a reduction-based approach to realize the idea of heuristic guidance. 
As summarized in \cref{alg:ouralg}, \algo takes a heuristic $\h:\SS\to\R$ and a base RL algorithm $\LL$ as input, and outputs an approximately optimal policy for the original MDP $\MM$.
During training, \algo iteratively runs the base algorithm $\LL$ to collect data from the MDP $\MM$ and then uses the heuristic $\h$ to modify the agent's collected experiences. 
Namely, in iteration $n$, the agent interacts with the original MDP $\MM$ and saves the raw transition tuples\footnote{If $\LL$ learns only with trajectories, we transform each tuple and assemble them to get the modified trajectory.} $\DD_n = \{(s,a,r,s')\}$ (line \ref{line:data collection}).  \algo\ then defines a reshaped MDP $\widetilde{\MM}_n \coloneqq (\SS,\AA,P,\widetilde{r}_n, \widetilde{\gamma}_n)$ (line \ref{line:reshape mdp}) by changing the rewards and lowering the discount factor:
\begin{align} \label{eq:reshaped MDP}
    \widetilde{r}_n(s,a) \coloneqq r(s,a) + (1-\lambda_n)\gamma \E_{s'|s,a}[\h(s')] \qquad \text{and}\qquad 
    \widetilde{\gamma}_n \coloneqq \lambda_n \gamma, 
\end{align}
where $\lambda_n\in[0,1]$ is the mixing coefficient. 
The new discount $\widetilde{\gamma}_n$ effectively gives  $\widetilde{\MM}_n$ a shorter horizon than $\MM$'s, 
while the heuristic $\h$ is blended into the new reward in \eqref{eq:reshaped MDP} 
to account for the missing long-term information.
We call $\widetilde{\gamma}_n =\lambda_n \gamma $ in \eqref{eq:reshaped MDP} the \emph{guidance discount} to be consistent with prior literature~\citep{jiang2015dependence}, which can be viewed in terms of our framework as using a zero heuristic.
%
%
%
In the last step (line \ref{line:train}), \algo calls the base algorithm $\LL$ to perform updates with respect to the reshaped MDP $\widetilde{\MM}_n$.
This is realized by
\begin{enumerate*}[label=\textit{\arabic*)}] 
    \item setting the discount factor used in $\LL$ to $\widetilde{\gamma}_n $, and 
    \item setting the sampled reward to $r + (\gamma  - \widetilde{\gamma}_n)  \h(s')$ for every transition tuple $(s,a,r,s')$ collected from $\MM$.
\end{enumerate*}
We remark that  the base algorithm $\LL$ in line \ref{line:data collection} always collects trajectories of lengths proportional to the original discount $\gamma$, while internally the optimization is done with a lower discount $\widetilde{\gamma}_n$ in line \ref{line:train}.

Over the course of training, \algo repeats the above steps with a sequence of increasing mixing coefficients $\{\lambda_n\}$.
From \eqref{eq:reshaped MDP} we see that as the agent interacts with the environment, the effects of the heuristic in MDP reshaping decrease and the effective horizon of the reshaped MDP increases.

\begin{algorithm}[t]
    \caption{\algofull (\algo)} \label{alg:ouralg}
\begin{algorithmic}[1]
    \REQUIRE MDP $\MM=(\SS,\AA,P,r,\gamma)$, RL algorithm $\LL$, heuristic $\h$,  mixing coefficients $\{\lambda_n\}$.
    
    \FOR{$n=1,\dots,N$}
    \STATE $\DD_n \gets \LL$.CollectData($\MM$) \label{line:data collection}
    \STATE Get $\lambda_n$  from $\{\lambda_n\}$ and construct $\widetilde{\MM}_n = (\SS,\AA,P,\widetilde{r}_n, \widetilde{\gamma}_n)$ according to \eqref{eq:reshaped MDP} using $\h$ and $\lambda_n$ \label{line:reshape mdp}
    \STATE $\pi_n \gets \LL$.Train($\DD_n$, $\widetilde{\MM}_n$) \label{line:train}
    \ENDFOR
    \RETURN $\pi_N$
\end{algorithmic}
\end{algorithm}

\vspace{-1mm}
\subsection{\algo as Horizon-based Regularization}
\vspace{-1mm}

We can think of \algo as introducing a horizon-based \emph{regularization} for RL, where the regularization center is defined by the heuristic and its strength diminishes as the mixing coefficient increases.
As the agent collects more experiences, \algo gradually removes the effects of regularization and the agent eventually optimizes for the original MDP. 


\algo's regularization is designed to reduce learning variance, similar to the role of regularization in supervised learning. 
Unlike the typical weight decay imposed on function approximators (such as the agent's policy or value networks), our proposed regularization leverages the structure of MDPs to regulate the complexity of the MDP the agent faces, which scales with the MDP's discount factor (or, equivalently, the horizon). When the guidance discount $\widetilde{\gamma}_n$ is lower than the original discount $\gamma$ (i.e. $\lambda_n<1$), the reshaped MDP $\widetilde{\MM}_n$ given by \eqref{eq:reshaped MDP} has a shorter horizon and requires fewer samples to solve.
However, the reduced complexity comes at the cost of bias, because the agent is now incentivized toward maximizing the performance with respect to the heuristic rather than the original long-term returns of $\MM$.
In the extreme case of $\lambda_n=0$, 
\algo would solve a zero-horizon contextual bandit problem with contexts (i.e. states) sampled from $d^\pi$ of $\MM$. 

%

\begin{figure*}[t]
	\centering
    \begin{subfigure}[b]{0.3\textwidth}
		\includegraphics[height=0.65\textwidth]{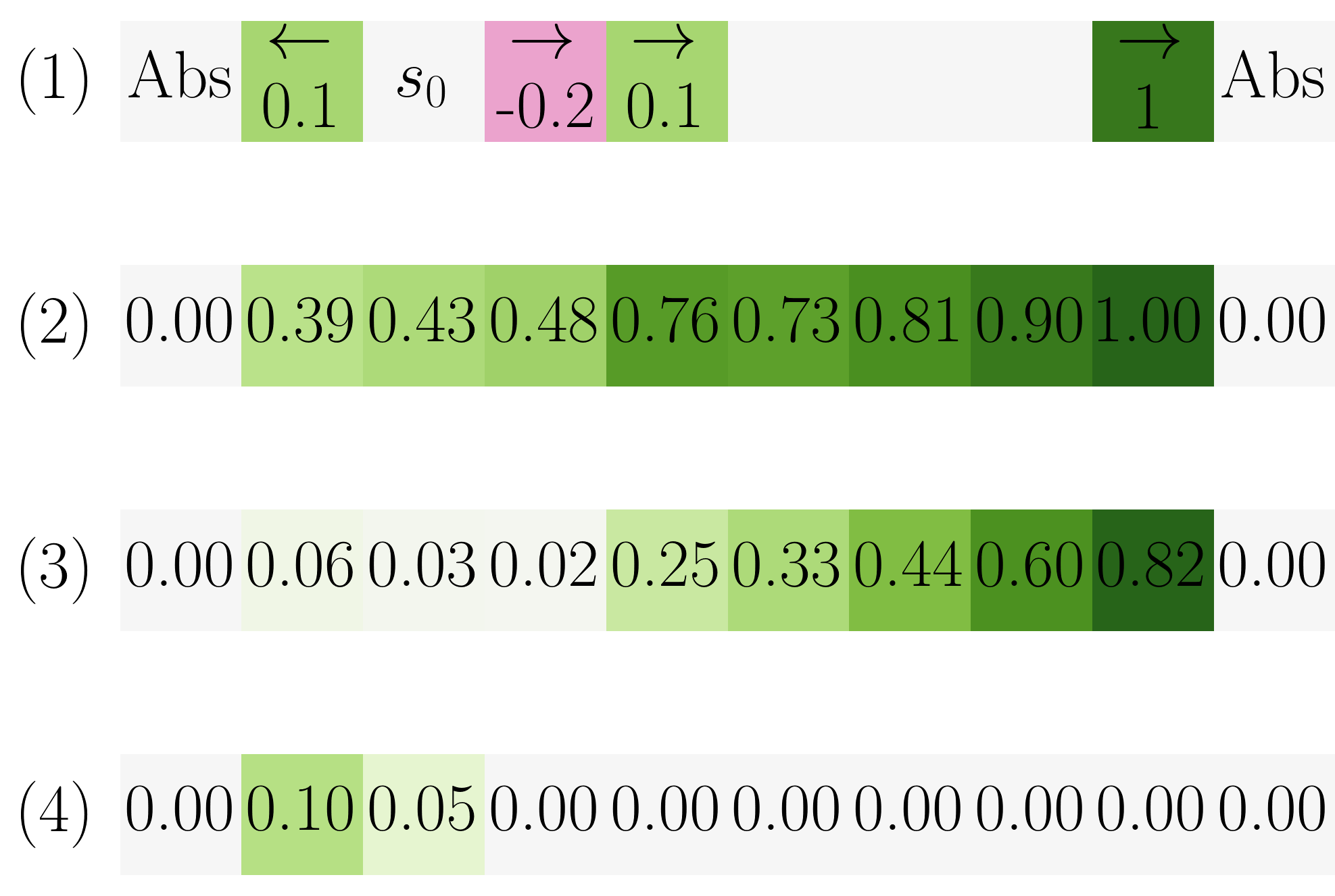}
		\caption{Heatmap of different values.}
		\label{fig:toy_value_heatmap}
	\end{subfigure}
	\hspace{2mm}
	\begin{subfigure}[b]{0.3\textwidth}
	    \centering
		\includegraphics[height=0.65\textwidth]{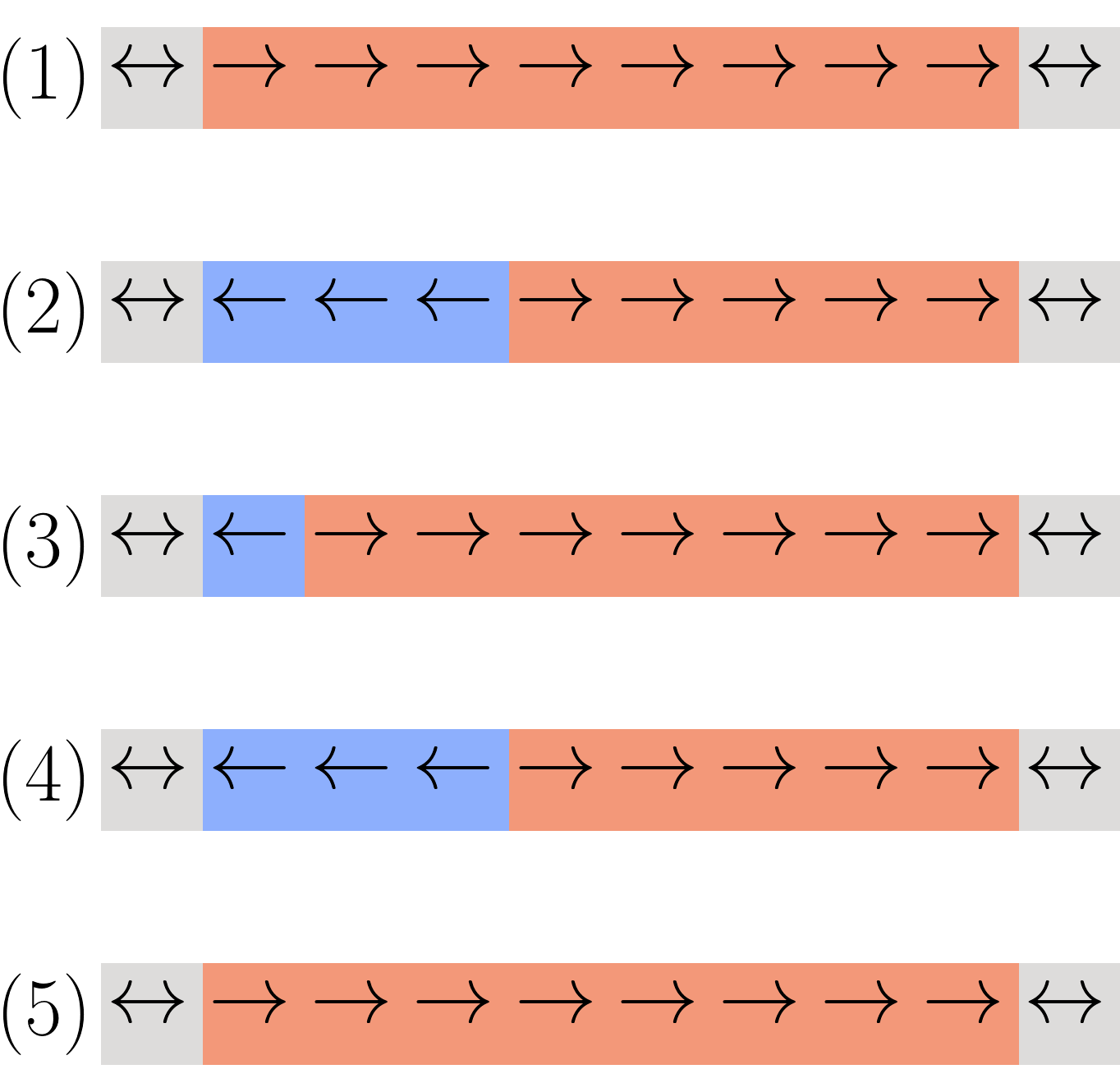}
		\caption{Different policy behaviors.}
		\label{fig:toy_policy_heatmap}
	\end{subfigure}
	\hspace{2mm}
	\begin{subfigure}[b]{0.3\textwidth}
		\includegraphics[height=0.65\textwidth]{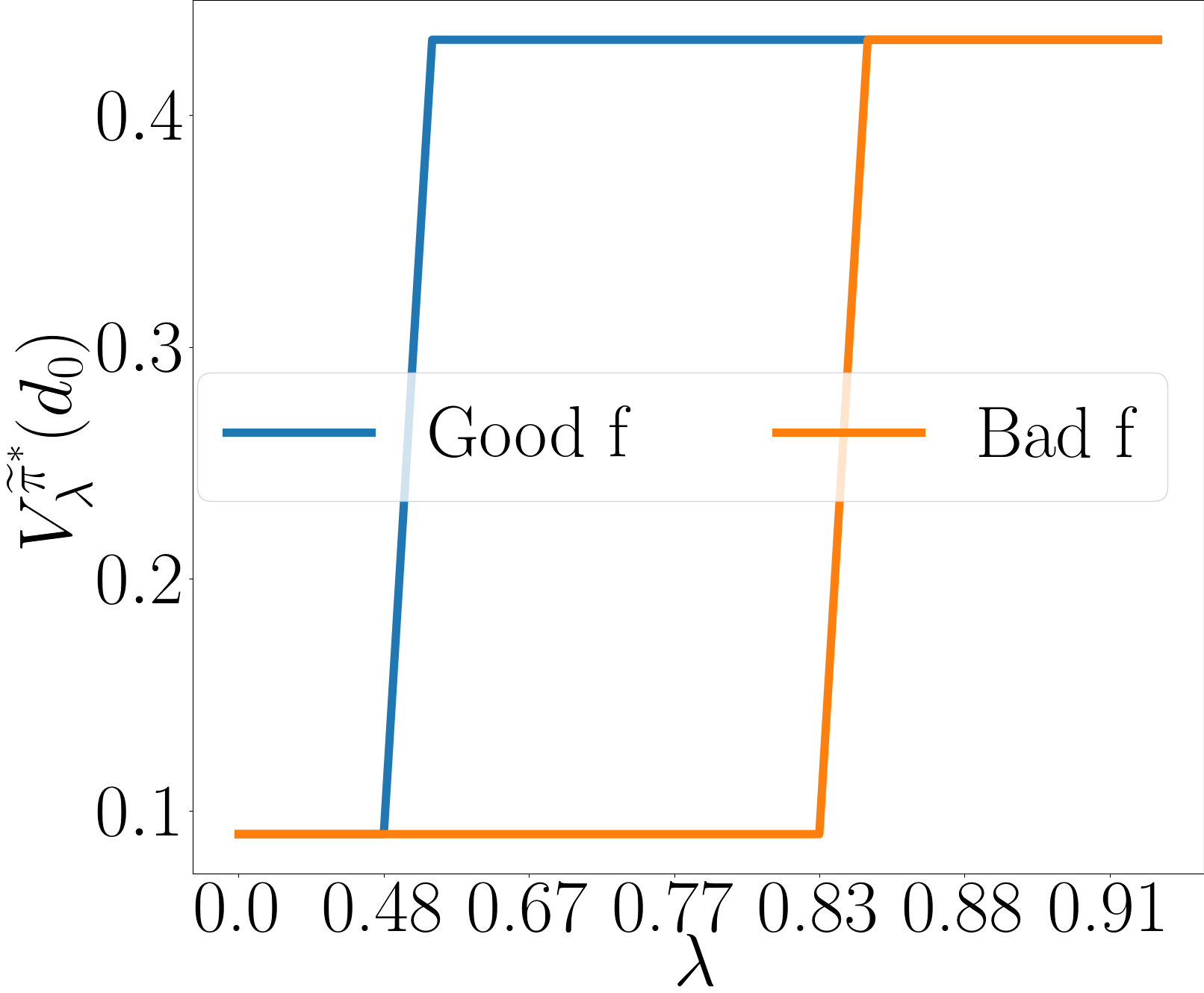}
		\caption{\algo with different $\h$ and $\lambda$.}
		\label{fig:toy_plot}
	\end{subfigure}
\caption{\small{\textbf{Example of \algo in a chain MDP.} Each cell in a row in each diagram represents a state from $\SS=\{1,\dots,10\}$. The agent starts at state $3$ ($s_0$), and states $1$ and $10$ are absorbing ($Abs$ in subfigure a-(1)). Actions $\AA=\{\leftarrow, \rightarrow\}$  move the agent left or right in the chain unless the agent is in an absorbing state. \textbf{Subfig. a-(1)} shows the reward function: $r(2,\leftarrow) =0.1, r(4,\rightarrow) =-0.2,r(5,\rightarrow)=0.1$, and all state-action pairs not shown in a-(1) yield $r=0$. \textbf{Subfig. a-(2)} shows $V^*$ for $\gamma = 0.9$. \textbf{Subfig. a-(3)} shows a good heuristic $\h$ --- $V(\text{random }\pi)$. \textbf{Subfig. a-(4)} shows a bad heuristic $\h$ --- $V(\text{myopic }\pi)$. \textbf{Subfig. b-(1)}: $\pi^*$ for $V^*$ from a-(2). \textbf{Subfig. b-(2)}: $\tilde{\pi}^*$ from \algo with $\h=0,\lambda=0.5$. \textbf{Subfig. b-(3)}:  $\tilde{\pi}^*$ from \algo with the good $\h$ from (a).(3) and $\lambda=0.5$. \textbf{Subfig. b-(4)}: $\tilde{\pi}^*$ from the bad $\h$ from a-(4), $\lambda=0.5$. \textbf{Subfig. b-(5)}: $\tilde{\pi}^*$ from the bad $\h$ and $\lambda=1$. \textbf{Subfig. (c) illustrates the takeaway message}: \emph{using \algo with a good $\h$  can find $\pi^*$ from $s_0$ even with a small $\lambda$ (see the $x$-axis), while \algo with a bad $\h$ requires a much higher $\lambda$ to discover $\pi^*$.}
}}
\label{fig:toy_example}
\end{figure*}

\vspace{-1mm}
\subsection{A Toy Example}
\vspace{-1mm}

We illustrate this idea in a chain MDP environment in \cref{fig:toy_example}. The optimal policy $\pi^*$ for this MDP's original $\gamma = 0.9$ always picks action $\rightarrow$, as shown in \cref{fig:toy_policy_heatmap}-(1), giving the optimal value $V^*$  in \cref{fig:toy_value_heatmap}-(2). 
Suppose we used a smaller guidance discount $\widetilde{\gamma} = 0.5\gamma$ to accelerate learning. This is equivalent to \algo with a zero heuristic $\h=0$ and $\lambda=0.5$. Solving this reshaped MDP yields a policy $\widetilde{\pi}^*$ that acts very myopically in the original MDP, as shown in \cref{fig:toy_policy_heatmap}-(2); the value function of $\widetilde{\pi}^*$ in the original MDP is visualized in \cref{fig:toy_value_heatmap}-(4). 

Now, suppose we use \cref{fig:toy_value_heatmap}-(4) as a heuristic in \algo instead of $\h=0$. This is a bad choice of heuristic (Bad $\h$)
as it introduces a large bias with respect to $V^*$ (cf. \cref{fig:toy_value_heatmap}-(2)). On the other hand, we can roll out a random policy in the original MDP and use its value function as the heuristic (Good $\h$), shown in \cref{fig:toy_value_heatmap}-(3). 
Though the random policy has an even \emph{lower} return at the initial state $s=3$, it gives a \emph{better} heuristic because this heuristic shares the same trend as $V^*$ in \cref{fig:toy_value_heatmap}-(1).
 \algo run with Good $\h$ and Bad $\h$ yields policies in \cref{fig:toy_policy_heatmap}-(3,4), and the quality of the resulting solutions in the original MDP, $V_{\lambda}^{\widetilde{\pi}^*}(d_0)$, is reported in \cref{fig:toy_plot} for different $\lambda$. Observe that \algo with a good heuristic can achieve $V^*(d_0)$ with a much smaller horizon $\lambda \le 0.5$. Using a bad $\h$ does not lead to $\pi^*$ at all when $\lambda = 0.5$ (\cref{fig:toy_policy_heatmap}-(4)) but is guaranteed to do so when $\lambda$ converges to $1$. (\cref{fig:toy_policy_heatmap}-(5)).

\vspace{-1mm}
\section{Theoretical Analysis}
\vspace{-1mm}

When can \algo accelerate learning?
Similar to typical regularization techniques, the horizon-based regularization of \algo leads to a bias-variance decomposition that can be optimized for better finite-sample performance compared to directly solving the original MDP.
However, a non-trivial trade-off is possible only when the regularization can bias the learning toward a good direction. 
In \algo's case this is determined by the heuristic, which resembles a prior we encode into learning.

In this section we provide \algo's theoretical foundation.
We first describe the bias-variance trade-off induced by \algo. Then we show how suboptimality in solving the reshaped MDP translates into performance in the original MDP, and identify the assumptions \algo needs the base RL algorithm to satisfy. 
In addition, we 
\rev{explain how \algo relates to PBRS}, and characterize the quality of heuristics and sufficient conditions for constructing good heuristics from batch data using offline RL. 

For clarity, we will focus on the reshaped MDP $\widetilde{\MM}=(\SS, \AA, P, \widetilde{r}, \widetilde{\gamma})$ for a fixed  $\lambda\in[0,1]$, where $\widetilde{r}, \widetilde{\gamma}$ are defined in \eqref{eq:short summary}.
We can view this MDP as the one in a single iteration of \algo.
For a policy $\pi$, we denote its state value function in $\widetilde{\MM}$ as $\widetilde{V}^\pi$, and the optimal policy and value function of $\widetilde{\MM}$ as $\widetilde{\pi}^*$ and $\widetilde{V}^*$, respectively.  
The missing proofs of the results from this section can be found in \cref{app:proofs}.

\vspace{-1mm}
\subsection{Short-Horizon Reduction: Performance Decomposition}
\vspace{-1mm}

Our main result is a performance decomposition, which characterizes how a heuristic $\h$ and suboptimality in solving the reshaped MDP $\widetilde{\MM}$ relate to performance in the original MDP $\MM$.

\begin{restatable}{theorem}{PerformanceDecomposition}
 \label{th:performance decomposition}
    For any policy $\pi$, heuristic $f:\SS\to\R$, and mixing coefficient $\lambda\in[0,1]$,
    \begin{align*}
        V^*(d_0) - V^\pi(d_0) &= \mathrm{Regret}(\h,\lambda,\pi) + \mathrm{Bias}(\h,\lambda,\pi) 
    \end{align*}
    where we define
    \begin{align}
        \mathrm{Regret}(\h,\lambda,\pi) &\coloneqq
        \lambda \left( \widetilde{V}^*(d_0) - \widetilde{V}^\pi(d_0) \right)  + \frac{1-\lambda}{1-\gamma}  \left( \widetilde{V}^*(d^{\pi}) - \widetilde{V}^\pi(d^{\pi}) \right)
        \label{eq:regret} \\
        \mathrm{Bias}(\h,\lambda,\pi) &\coloneqq
        \left( V^*(d_0) - \widetilde{V}^*(d_0) \right) + \frac{\gamma(1-\lambda)}{1-\gamma} \E_{s,a\sim d^\pi}   \E_{s'|s,a} \left[ \h(s')  - \widetilde{V}^*(s')\right]  \label{eq:bias} 
    \end{align}
    Furthermore, $\forall b\in\R$, $\mathrm{Bias}(\h,\lambda,\pi) = \mathrm{Bias}(\h+b,\lambda,\pi)$ 
    and 
    $\mathrm{Regret}(\h,\lambda,\pi) = \mathrm{Regret}(\h+b,\lambda,\pi)$.
\end{restatable}
The theorem shows that suboptimality of a policy $\pi$ in the original MDP $\MM$  can be decomposed into 
\begin{enumerate*} [label=\textit{\arabic*)}]
    \item  a \emph{bias} term due to solving a reshaped MDP $\widetilde{\MM}$ instead of the original MDP $\MM$, and 
    \item a \emph{regret} term (i.e. the learning variance) due to $\pi$ being suboptimal in the reshaped MDP $\widetilde{\MM}$.
\end{enumerate*}
Moreover, it shows that heuristics are equivalent up to constant offsets. In other words, only the relative ordering between states that a heuristic induces matters in learning, not the absolute values.

Balancing the two terms trades off bias and variance in learning.
Using a smaller $\lambda$ 
replaces the long-term information with the heuristic and make the horizon of the reshaped MDP $\widetilde{\MM}$ shorter. 
Therefore, given a finite interaction budget, the regret term in \eqref{eq:regret} can be more easily minimized, though the bias term in \eqref{eq:bias} can potentially be large if the heuristic is bad. 
On the contrary, with $\lambda=1$, the bias is completely removed, as the agent solves the original MDP $\MM$ directly.
%

\vspace{-1mm}
\subsection{Regret,  Algorithm Requirement, and Relationship with PBRS}
\vspace{-1mm}

The regret term in \eqref{eq:regret} characterizes the performance gap due to $\pi$ being suboptimal in the reshaped MDP $\widetilde{\MM}$, because    $\mathrm{Regret}(\h,\lambda,\widetilde{\pi}^*) = 0$ for any $\h$ and $\lambda$.
For learning, we need the base RL algorithm $\LL$ to find a policy $\pi$ such that the regret term in \eqref{eq:regret} is small. 
By the definition in \eqref{eq:regret}, the base RL algorithm $\LL$ is required not only to find a policy $\pi$ such that   $\widetilde{V}^*(s) - \widetilde{V}^\pi(s)$ is small for states from $d_0$,
\emph{but also for states $\pi$ visits when rolling out in the original MDP $\MM$}.
In other words, it is  insufficient for the base RL algorithm to only optimize for $\widetilde{V}^\pi(d_0)$ (the performance in the reshaped MDP with respect to the initial state distribution; \rev{see \cref{sec:RL with heuristic setup}}).
For example, suppose $\lambda=0$ and $d_0$ concentrates on a single state $s_0$. Then maximizing $\widetilde{V}^\pi(d_0)$ alone would only optimize $\pi(\cdot|s_0)$ and the policy $\pi$ need not know how to act in other parts of the state space.

To use \algo, we need the base algorithm to learn a policy $\pi$ that has small \emph{action gaps} in the reshaped MDP $\widetilde{\MM}$ \emph{but along trajectories in the original MDP $\MM$}, as we show below. This property is satisfied by off-policy RL algorithms such as Q-learning~\citep{jin2018q}.

\begin{restatable}{proposition}{RegretAsActionGap} \label{th:regret as action gap}
 For \emph{any} policy $\pi$, heuristic $f:\SS\to\R$ and mixing coefficient $\lambda\in[0,1]$,
    \begin{align*}
    \textstyle
        \mathrm{Regret}(\h,\lambda,\pi)
        =  - \E_{\rho^\pi(d_0)} \left[ \sum_{t=0}^\infty \gamma^t  \widetilde{A}^*(s_t,a_t) \right]
    \end{align*}
    where $\rho^\pi(d_0)$ denotes the trajectory distribution of running $\pi$ from $d_0$, and 
    $\widetilde{A}^*(s,a) = \widetilde{r}(s,a)+\widetilde{\gamma}\E_{s'|s,a}[\widetilde{V}^*(s')] -  \widetilde{V}^*(s) \leq 0 $ is the action gap 
    with respect to the optimal policy $\widetilde{\pi}^*$ of $\widetilde{\MM}$.
\end{restatable}

Another way to comprehend the regret term is through studying its dependency on $\lambda$.
When $\lambda=1$, $\mathrm{Regret}(\h,0,\pi) = V^*(d_0) - V^\pi(d_0)$, which is identical to the policy regret in $\MM$ for a \emph{fixed} initial distribution $d_0$.
On the other hand, when $\lambda=0$, 
$
\mathrm{Regret}(\h,0,\pi)=
        \max_{\pi'} \frac{1}{1-\gamma}  \E_{s\sim d^{\pi}}[ \widetilde{r}(s,\pi') - \widetilde{r}(s,\pi)] 
$, which is the regret of a \emph{non-stationary} contextual bandit problem where the context distribution is $d^\pi$ (the average state distribution of $\pi$).
In general, for $\lambda\in(0,1)$, the regret notion mixes a short-horizon non-stationary problem and a long-horizon stationary problem.

One natural question is whether the reshaped MDP $\widetilde{\MM}$ has a more complicated and larger value landscape than the original MDP $\MM$, because these characteristics may affect the regret rate of a base algorithm.
We show that $\widetilde{\MM}$ preserves the value bounds and linearity of the original MDP. 
\begin{restatable}{proposition}{PreservedProperties} \label{th:preserved MDP properties}
    Reshaping the MDP as in \eqref{eq:short summary} preserves the following characteristics:
    \begin{enumerate*}[label=\textit{\arabic*)}]
        \item If $\h(s) \in [0,\frac{1}{1-\gamma}]$, then $\widetilde{V}^\pi(s)   \in [0,\frac{1}{1-\gamma}]$ for all $\pi$ and $s\in\SS$.
        \item If $\widetilde{\MM}$ is a linear MDP with feature vector  $\phi(s,a)$ (i.e. $r(s,a)$ and $\E_{s'|s,a}[g(s')]$ for any $g$ can be linearly parametrized in $\phi(s,a)$), then $\widetilde{\MM}$ is also a linear MDP with feature vector $\phi(s,a)$.
    \end{enumerate*}
\end{restatable}
\vspace{-1mm}

\rev{
On the contrary, the MDP $\overline{\MM}_\lambda \coloneqq (\SS,\AA,P,\overline{r},\lambda\gamma)$ in \cref{sec:RL with heuristic setup} does not have these properties. We can show that $\overline{\MM}_\lambda$ is equivalent to $\widetilde{\MM}$ up to a PBRS transformation (i.e., $\bar{r}(s,a) = \tilde{r}(s,a)+ \tilde{\gamma} \E_{s'|s,a}[\h(s')]-\h(s)$). 
Thus, \algo~incorporates guidance discount into PBRS 
with nicer properties. 
}


\subsection{Bias and Heuristic Quality}
\label{sec:heuristic_quality}
\vspace{-1mm}

The bias term in \eqref{eq:bias} characterizes suboptimality due to using a heuristic $\h$ in place of long-term state values in $\MM$. 
What is the best heuristic in this case?
From the definition of the bias term in \eqref{eq:bias}, we see that the ideal heuristic is the optimal value $V^*$, as  $ \mathrm{Bias}(V^*,\lambda,\pi) = 0$ for all $\lambda\in[0,1]$. 
By continuity, we can expect that if $\h$ deviates from $V^*$ a little, then the bias is small.
\begin{restatable}{corollary}{LinfBiasBound}\label{th:l-inf bias bound}
    If  $\inf_{b\in\R} \|\h + b - V^*\|_\infty \leq \epsilon$, then $\mathrm{Bias}(\h,\lambda,\pi) 
    \leq \frac{(1-\lambda\gamma)^2}{(1-\gamma)^2} \epsilon
    $.
\end{restatable}

To better understand how the heuristic $\h$ affects the bias, we derive an upper bound on the bias by replacing the first term in \eqref{eq:bias} with an upper bound that depends only on $\pi^*$.
 \begin{restatable}{proposition}{BiasBound} \label{th:bias bound}
For $g:\SS\to\R$ and $\eta\in[0,1]$, define $\CC(\pi,g,\eta) \coloneqq \E_{\rho^{\pi}(d_0)} \left[ \sum_{t=1}^\infty \eta^{t-1}  g(s_t) \right] $. 
    Then 
    $
        \mathrm{Bias}(\h,\lambda,\pi) \leq
         (1-\lambda)\gamma (  \CC(\pi^*,V^*-\h, \lambda\gamma) + 
         \CC(\pi,\h-\widetilde{V}^*, \gamma)
         )
    $.
\end{restatable}
In \cref{th:bias bound}, the term $(1-\lambda)\gamma \CC(\pi^*,V^*-\h, \lambda\gamma)$ is the underestimation error of the heuristic $\h$ under the states visited by the optimal policy $\pi^*$ in the original MDP $\MM$.
Therefore, to minimize the first term in the bias, we would want the heuristic $\h$ to be large along the paths that $\pi^*$ generates.

However, \cref{th:bias bound} also discourages the heuristic from being arbitrarily large, because the second term in the bias in \eqref{eq:bias} (or, equivalently, the second term in  \cref{th:bias bound}) incentivizes the heuristic to underestimate the optimal value of the reshaped MDP $\widetilde{V}^*$.
More precisely, the second term requires the heuristic to obey some form of spatial consistency. A quick intuition is the observation that if $\h(s)=V^{\pi'}(s)$ for some $\pi'$ or $\h(s)=0$, then $\h(s) \leq  \widetilde{V}^*(s)$ for all $s\in\SS$.
More generally, we show that if the heuristic is \emph{improvable} with respect to the original MDP $\MM$ (i.e. the heuristic value is lower than that of the max  of Bellman backup), then $\h(s) \leq  \widetilde{V}^*(s)$.
By \cref{th:bias bound}, learning with an improvable heuristic in \algo has a much smaller bias. 
\begin{definition}
    Define the Bellman operator $(\BB\h)(s,a) \coloneqq r(s,a) +\gamma \E_{s'|s,a}[\h(s')]$. 
    A heuristic function $\h:\SS\to\R$ is said to be \emph{improvable} with respect to an MDP $\MM$ if $\max_a (\BB \h)(s,a) \geq \h(s)$.
\end{definition}
\begin{restatable}{proposition}{ImprovableHeuristic} \label{th:improvable heuristic}
    If $\h$ is improvable with respect to $\MM$, then $\widetilde{V}^*(s) \geq \h(s) $, for all $\lambda\in[0,1]$.
\end{restatable}

%

\vspace{-1mm}
\subsection{Pessimistic Heuristics are Good Heuristics }
\vspace{-1mm}

While \cref{th:l-inf bias bound} shows that \algo can handle  an imperfect heuristic, this result is not ideal. The corollary depends on the $\ell_\infty$ approximation error, which can be difficult to control in large state spaces.
Here we provide a more refined sufficient condition of good heuristics. We show that the concept of \emph{pessimism} in the face of uncertainty provides a finer mechanism for controlling the approximation error of a heuristic and would allow us to remove the $\ell_\infty$-type error.
This result is useful for constructing heuristics from data that does not have sufficient support.

From \cref{th:bias bound} we see that the source of the $\ell_\infty$ error is the second term in the bias upper bound, as it depends on the states that the agent's policy visits which can change during learning.
To remove this dependency, we can use improvable heuristics (see \cref{th:improvable heuristic}), as they satisfy $\h(s) \leq  \widetilde{V}^*(s)$.
Below we show that Bellman-consistent pessimism  
yields improvable heuristics.

\begin{restatable}{proposition}{BellmanPessimismAndImprovableHeuristic}  \label{th:bellman pessimism and improvable heuristic}
    Suppose $\h(s) = Q(s,\pi')$ for some policy $\pi'$ and function $Q:\SS\times\AA\to\R$ such that 
    $    Q(s,a) \leq (\BB \h) (s,a) $,
    $\forall s\in\SS$, $a\in\AA$.
    Then $\h$ is improvable and  $f(s)\leq V^{\pi'}(s)$ for all $s\in\SS$.
\end{restatable}

The Bellman-consistent pessimism in \cref{th:bellman pessimism and improvable heuristic} essentially says that $\h$ is pessimistic with respect to the Bellman backup.
This condition has been used as the foundation for designing pessimistic off-policy RL algorithms, such as pessimistic value iteration~\citep{jin2020pessimism} and algorithms based on pessimistic absorbing MDPs~\citep{liu2020provably}.
In other words, these pessimistic algorithms can be used to construct good heuristics with small bias in \cref{th:bias bound} from offline data.
With such a heuristic, the bias upper bound would be simply 
$\mathrm{Bias}(\h,\lambda,\pi) \leq
 (1-\lambda)\gamma  \CC(\pi^*,V^*-\h, \lambda\gamma)$. 
Therefore, as long as enough batch data are sampled from a distribution that covers states that $\pi^*$ visits, 
these pessimistic algorithms can construct good heuristics with nearly zero bias for \algo with high probability.






\vspace{-1mm}
\section{Experiments} \label{sec:exps}
\vspace{-1mm}

\def\sac{SAC\xspace}
\def\sacw{SAC w/ BC\xspace}
\def\zeroalgo{\algo-zero\xspace}
\def\mcalgo{\algo-MC\xspace}
\def\bc{BC\xspace}

We validate our framework \algo experimentally in MuJoCo (commercial license)~\citep{todorov2012mujoco} robotics control problems and Procgen games (MIT License)~\citep{procgen}, 
where soft actor critic (SAC)~\cite{haarnoja2018soft} and proximal policy optimization (PPO)~\cite{schulman2017proximal} were used as the base RL algorithms, respectively\footnote{Code to replicate all experiments is available at \href{https://github.com/microsoft/HuRL}{https://github.com/microsoft/HuRL}.}.  
The goal is to study whether \algo can accelerate learning by shortening the horizon with heuristics.
In particular, we conduct studies to investigate the effects of different heuristics and mixing coefficients. 
%
%
Since the main focus here is on the possibility of leveraging a \emph{given} heuristic to accelerate RL algorithms, in these experiments we used vanilla techniques 
to construct heuristics for \algo. 
Experimentally studying the design of heuristics for a domain or a batch of data is beyond the scope of the current paper but are important future research directions. 
%
For space limitation, here we report only the results of the MuJoCo experiments. The results on Procgen games along with other experimental details can also be found in \cref{app:experiments}.
%

\vspace{-1mm}
\subsection{Setup} \label{ssec:setup}
\vspace{-1mm}

We consider four MuJoCo environments with dense rewards (Hopper-v2, HalfCheetah-v2, Humanoid-v2, and Swimmer-v2) \rev{and a sparse reward version of Reacher-v2 (denoted as Sparse-Reacher-v2)\footnote{The reward is zero at the goal and $-1$ otherwise.}.}
\rev{We design the experiments to simulate two learning scenarios.
First, we use Sparse-Reacher-v2 to simulate the setting where an engineered heuristic based on domain knowledge is available; since this is a goal reaching task, we designed a heuristic $\h(s) = r(s,a) - 100 \norm{e(s)-g(s)} $, where $e(s)$ and $g(s)$ denote the robot's end-effector position and the goal position, respectively.
Second, we use the dense reward environments to model scenarios} where a batch of data collected by multiple behavioral policies is available before learning, and a heuristic is constructed by an offline policy evaluation algorithm from the batch data \rev{(see Appendix~\ref{app:details} for details). In brief}, we generated these behavioral policies by running \sac from scratch and saved the intermediate policies generated in training.
We then use least-squares regression to fit a 
neural network to predict empirical Monte-Carlo returns of the trajectories in the sampled batch of data.
We also use behavior cloning (\bc) to warm-start \rev{all} RL agents based on the same batch dataset \rev{in the dense reward experiments.}

The base RL algorithm here, SAC, is based on the standard implementation in Garage (MIT License)~\cite{garage}.
The policy and value networks are fully connected independent neural networks. The policy is Tanh-Gaussian and the value network has a linear head. 

\vspace{-2mm}
\paragraph{Algorithms.}
We compare the performance of different algorithms below.
\begin{enumerate*}[label=\textit{\arabic*)}]    
\item \bc
\item SAC 
\item SAC with \bc warm start (\sacw)
\item \rev{\algo with the engineered heuristic (\algo)}
\item \algo with a zero heuristic and \bc warm start (\zeroalgo)
\item \algo with the Monte-Carlo heuristic and \bc warm start (\mcalgo)
\item \rev{SAC with PBRS reward (and \bc warm start, if applicable) (PBRS)}.
\end{enumerate*}
For the \algo algorithms, the mixing coefficient was scheduled as $\lambda_n = \lambda_0 + (1-\lambda_0) c_\omega
\tanh(\omega (n-1) )$, for $n=1,\dots,N$, where $\lambda_0 \in [0,1]$, $\omega>0$ controls the increasing rate, and $c_\omega$ is a normalization constant such that $\lambda_\infty=1$ and $\lambda_n \in[0,1]$.
We chose these algorithms to study the effect of each additional warm-start component (\bc and heuristics) added on top of vanilla \sac.
\zeroalgo is \sacw but with an extra $\lambda$ schedule above that further lowers the discount, whereas \sac and \sacw keep a constant discount factor.

\vspace{-2mm}
\paragraph{Evaluation and Hyperparameters.}
In each iteration, the RL agent has a fixed sample budget for environment interactions, and its performance is measured in terms of undiscounted cumulative returns of the deterministic mean policy extracted from \sac.
The hyperparameters used in the algorithms above were selected as follows. First, the learning rates and the discount factor of the base RL algorithm, \sac, were tuned for each environment. The tuned discount factor was used as the discount factor $\gamma$ of the original MDP $\MM$.
Fixing the hyperparameters above, we additionally tune $\lambda_0$ and $\omega$ for the $\lambda$ schedule of \algo for each environment and each heuristic.
Finally, after all these hyperparameters were fixed, we conducted additional testing runs with 30 different random seeds and report their statistics here. 
Sources of randomness included the data collection process of the behavioral policies, training the heuristics from batch data, BC, and online RL. However, the behavioral policies were fixed across all testing runs. 
\rev{
We chose this hyperparameter tuning procedure to make sure that the baselines (i.e. SAC) compared in these experiments are their best versions. 
}

\begin{figure*}[t]
	\centering
	\begin{subfigure}{0.29\textwidth}
		\includegraphics[width=\textwidth]{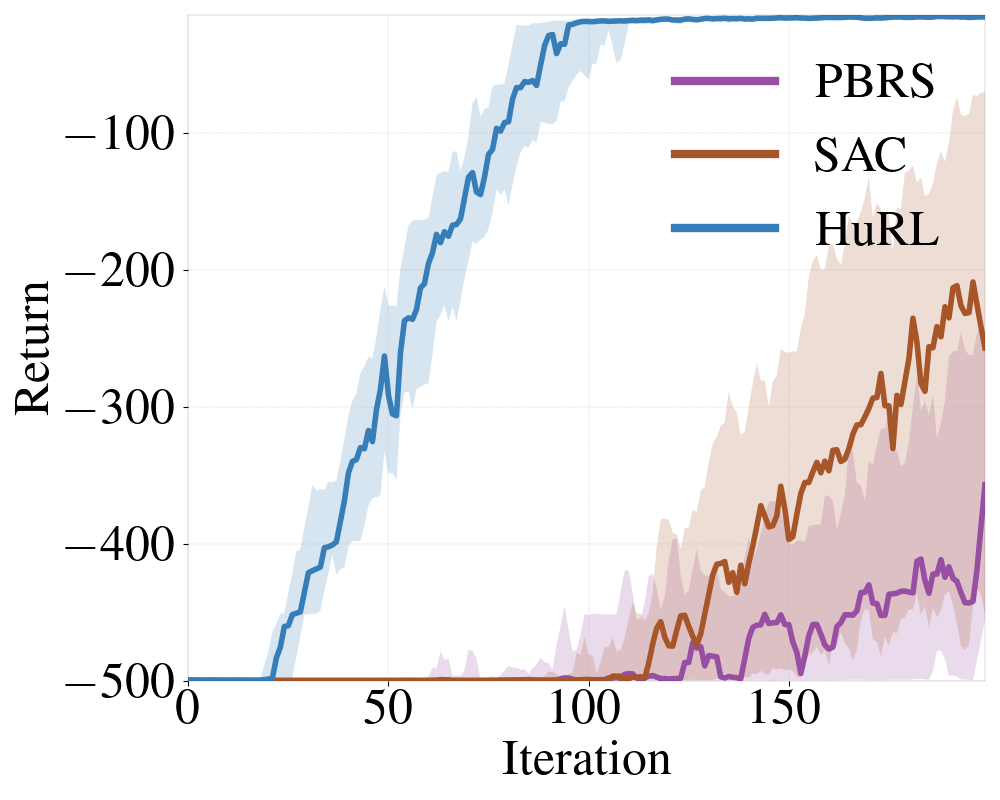}
		\caption{Sparse-Reacher-v2}
		\label{fig:sparse reacher}
	\end{subfigure}
	\begin{subfigure}{0.29\textwidth}
		\includegraphics[width=\textwidth]{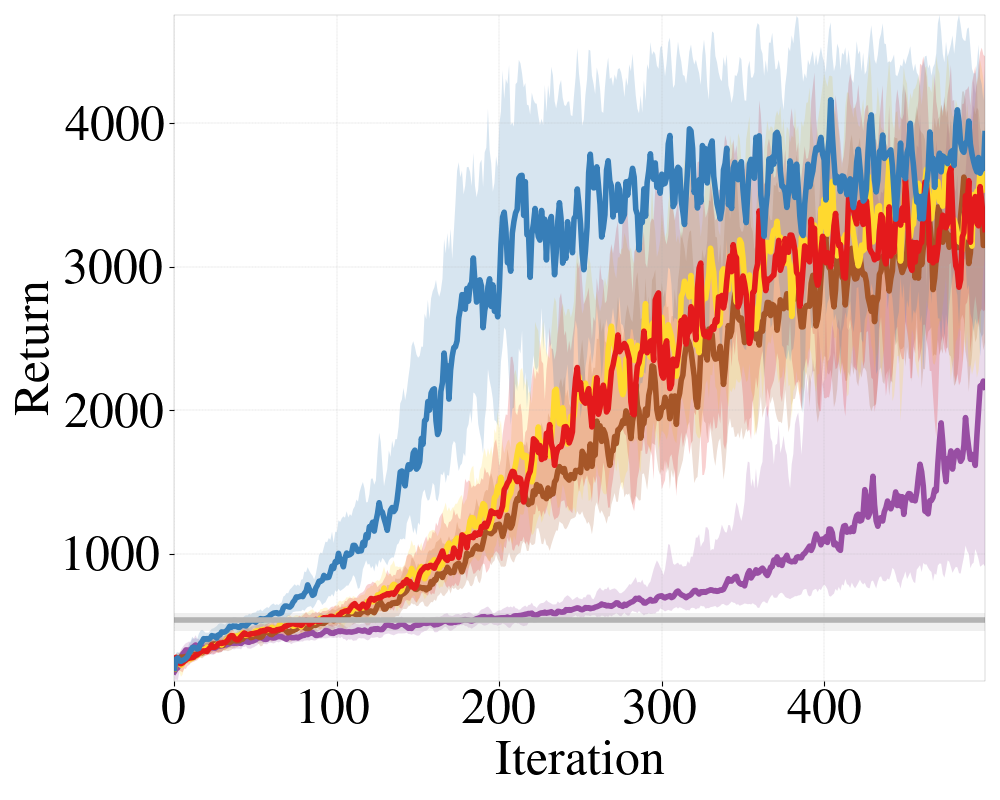}
		\caption{Humanoid-v2}
		\label{fig:main humanoid}
	\end{subfigure}
    \begin{subfigure}{0.29\textwidth}
		\includegraphics[width=\textwidth]{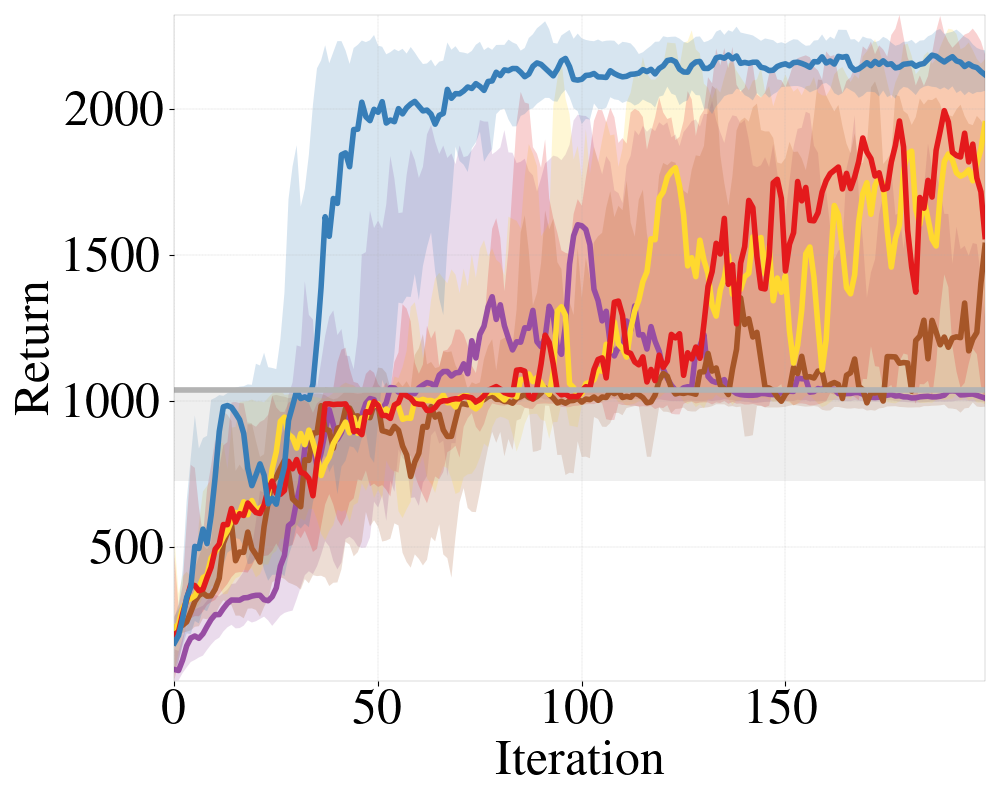}
		\caption{Hopper-v2}
		\label{fig:main hopper}
	\end{subfigure}
	\begin{subfigure}{0.29\textwidth}
		\includegraphics[width=\textwidth]{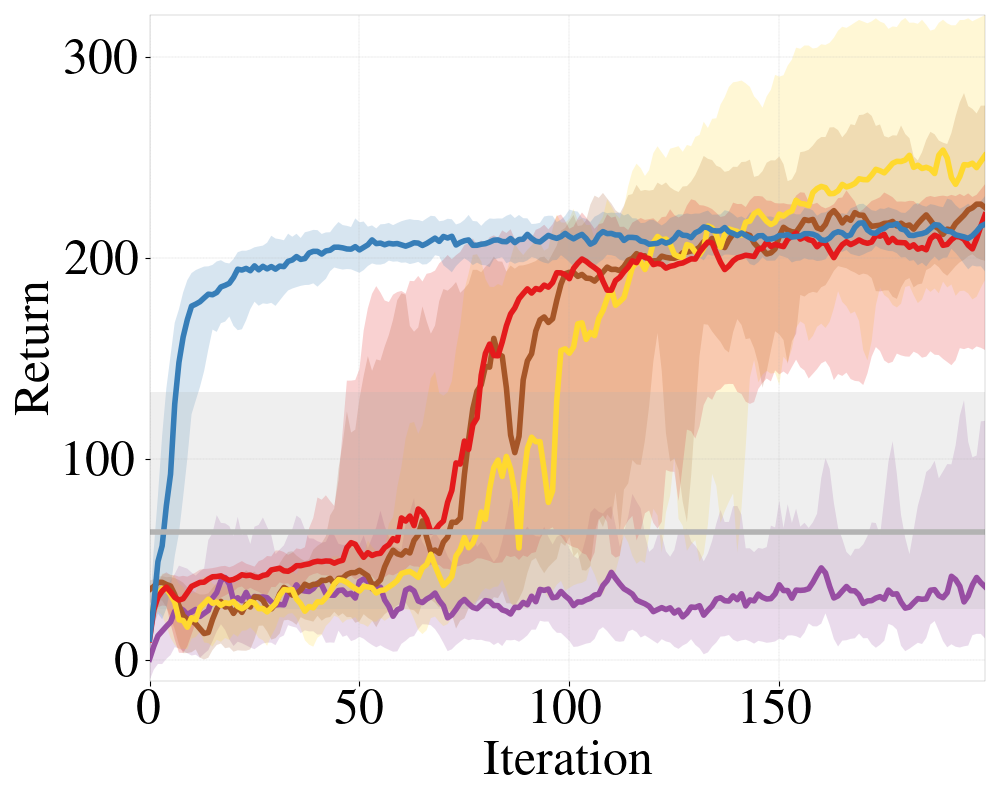}
		\caption{Swimmer-v2}
		\label{fig:main swimmer}
	\end{subfigure}
	\begin{subfigure}{0.29\textwidth}
		\includegraphics[width=\textwidth]{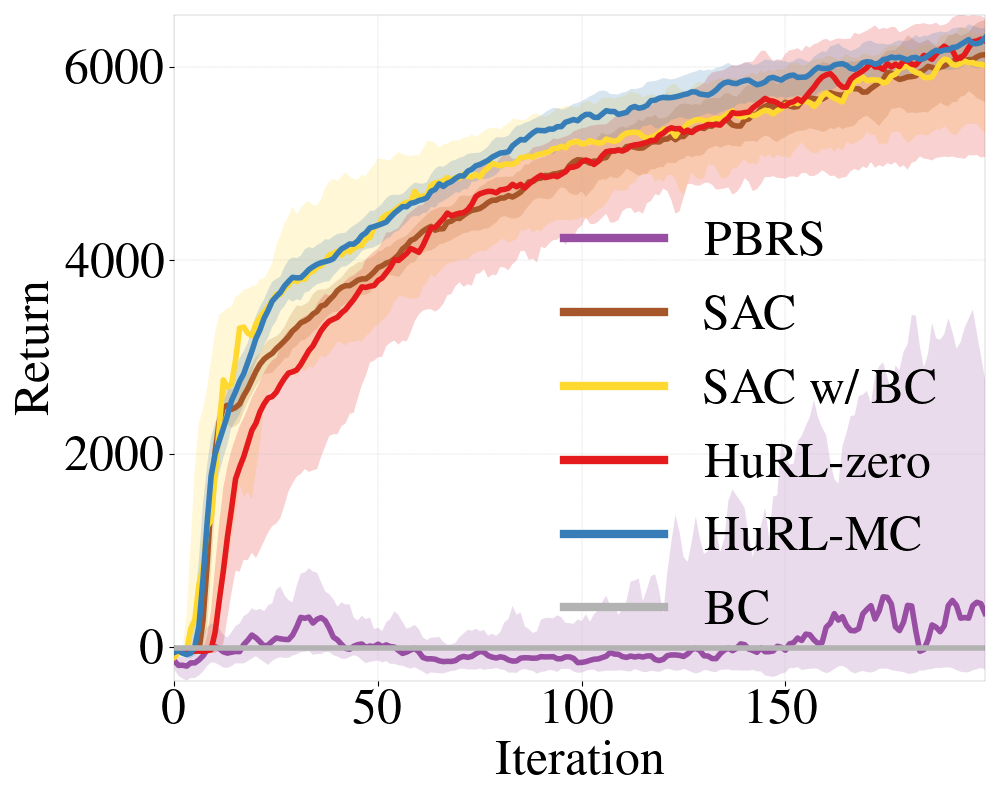}
		\caption{HalfCheetah-v2}
		\label{fig:main halfcheetah}
	\end{subfigure}
	\begin{subfigure}{0.29\textwidth}
        \includegraphics[width=\textwidth]{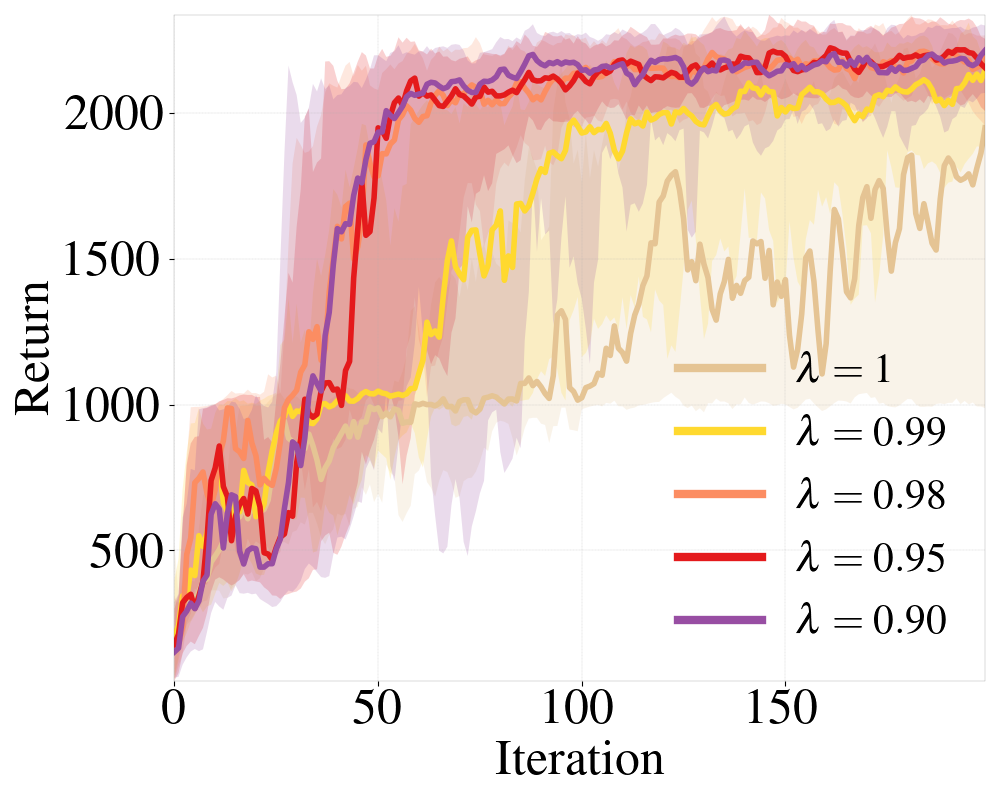}
        \caption{ $\lambda_0$ ablation.}
        \label{fig:lambda ablation (hopper)}
	\end{subfigure}
\caption{\small{Experimental results.
(a) uses an engineered heuristic for a sparse reward problem;
(b)-(e) use heuristics learned from offline data and share the same legend; (e) shows ablation results of different initial $\lambda_0$ in Hopper-v2.
The plots show the $25$th, $50$th, $75$th percentiles of algorithm performance over 30 random seeds.}}
\label{fig:main results}
\vspace{-2mm}
\end{figure*}



\vspace{-1mm}
\subsection{Results Summary}
\vspace{-1mm}

\cref{fig:main results} shows the results on the MuJoCo environments. \rev{Overall, we see that \algo is able to leverage engineered and learned heuristics to significantly improve the learning efficiency. This trend is consistent across all environments that we tested on.}

\rev{For the sparse-reward experiments, we see that SAC and PBRS struggle to learn, while \algo~is able to converge to the optimal performance much faster.
For the dense reward experiments, similarly \mcalgo converges much faster, though the gain in HalfCheetah-v2 is minor and it might have converged to a worse local maximum in Swimmer-v2.
}
In addition, we see that warm-starting \sac using BC (i.e. \sacw) can improve the learning efficiency compared with the vanilla \sac, but using \bc alone does not result in a good policy. Lastly, we see that using the zero heuristic (\zeroalgo) with extra $\lambda$-scheduling does not further improve the performance of \sacw. This comparison verifies that the learned Monte-Carlo heuristic provides non-trivial information.

\rev{
Interestingly, we see that applying PBRS to SAC leads to even worse performance than running SAC with the original reward. There are two reasons why SAC+PBRS is less desirable than SAC+HuRL as we discussed before: 
\begin{enumerate*}[label=\textit{\arabic*)}]    
\item
PBRS changes the reward/value scales in the induced MDP, and popular RL algorithms like SAC are very sensitive to such changes. In contrast, \algo~induces values on the same scale as we show in \cref{th:preserved MDP properties}. 
\item In \algo, we are effectively providing the algorithm some more side-information to let SAC shorten the horizon when the heuristic is good.
\end{enumerate*} 
}

The results in \cref{fig:main results} also have another notable aspect. 
Because the datasets used in the dense reward experiments contain trajectories collected by a range of policies, it is  likely that \bc suffers from disagreement in action selection among different policies. 
Nonetheless, training a heuristic using a basic Monte-Carlo regression seems to be less sensitive to these conflicts and still results in a helpful heuristic for \algo. 
One explanation can be that heuristics are only functions of states, not of states and actions, and therefore the conflicts are minor. 
%
%
%
Another plausible explanation is that \algo only uses the heuristic to \emph{guide} learning, and does not completely rely on it to make decisions 
%
Thus, \algo can be more robust to the heuristic quality, or, equivalently, to the quality of prior knowledge.




\vspace{-1mm}
\subsection{Ablation: Effects of Horizon Shortening}
\vspace{-1mm}
%
To further verify that the acceleration in \cref{fig:main results} is indeed due to horizon shortening, we conducted an ablation study for \mcalgo on Hopper-v2, whose results are presented in \cref{fig:lambda ablation (hopper)}. 
\mcalgo's best $\lambda$-schedule hyperparameters on Hopper-v2, which are reflected in its performance in the aforementioned \cref{fig:main hopper}, induced a near-constant schedule at $\lambda=0.95$; to obtain the curves in \cref{fig:lambda ablation (hopper)}, we ran \mcalgo with constant-$\lambda$ schedules for several more $\lambda$ values.
%
\cref{fig:lambda ablation (hopper)} shows that increasing $\lambda$ above $0.98$ leads to a performance drop. Since using a large $\lambda$ decreases bias and makes the reshaped MDP more similar to the original MDP, we conclude that the increased learning speed on Hopper-v2 is due to \algo's horizon shortening (coupled with the guidance provided by its heuristic).

\vspace{-1mm}
\section{Related Work}
\label{sec:related}
\vspace{-1mm}


%

\vspace{-2mm}
\paragraph{Discount regularization.}
The horizon-truncation idea can be traced back to Blackwell optimality in the known MDP setting~\citep{blackwell1962discrete}.
%
Reducing the discount factor amounts to running \algo with a zero heuristic. 
\citet{petrik2008biasing,jiang2015dependence,jiang2016structural} study the MDP setting; \citet{chen2018improving} study  POMDPs. \citet{amit2020discount} focus on discount regularization for Temporal Difference (TD) methods, while \citet{NEURIPS2019_eba237ec} use a logarithmic mapping 
to lower the discount for online RL.

\vspace{-2mm}
\paragraph{Reward shaping.}
Reward shaping has a long history in RL, from the seminal PBRS work~\citep{ng1999policy} to recent bilevel-optimization approaches~\citep{hu2020learning}. 
\citet{tessler2020maximizing} consider a complementary problem to \algo: given a discount $\gamma'$, they find a reward $r'$ that preserves trajectory ordering in the original MDP.
Meanwhile there is a vast literature on bias-variance trade-off for online RL with horizon truncation. TD($\lambda$)~\citep{seijen2014true,efroni2018beyond} and Generalized Advantage Estimates~\citep{schulman2015high} blend value estimates across discount factors, while~\citet{sherstan2020gamma} use the discount factor as an input to the value function estimator. TD($\Delta$)~\citep{romoff2019separating} computes differences between value functions across discount factors. 

\paragraph{Heuristics in model-based methods.}
Classic uses of heuristics include A*~\citep{hart1968formal}, Monte-Carlo Tree Search (MCTS)~\citep{browne2012survey}, and Model Predictive Control (MPC)~\citep{richalet1978model}.
\citet{zhong2013value} shorten the horizon in MPC using a value function approximator. 
\citet{hoeller2020deep} additionally use an estimate for the running cost to trade off solution quality and amount of computation. 
\citet{bejjani2018planning} show heuristic-accelerated truncated-horizon MPC on actual robots and tune the value function throughout learning. 
\citet{bhardwaj2020blending} augment MPC with a terminal value heuristic, which can be viewed as an instance of \algo where the base algorithm is MPC.
\rev{\citet{asai2020} learn an MDP expressible in the STRIPS formalism that can  benefit from relaxation-based planning heuristics.}
\rev{But \algo is more general, as it does not assume model knowledge and can work in unknown environments.}

\vspace{-2mm}
\paragraph{Pessimistic extrapolation.} 
Offline RL techniques employ pessimistic extrapolation for robustness~\citep{jin2020pessimism}, and their learned value functions can be used as heuristics in \algo. 
\citet{kumar2020conservative} penalize out-of-distribution actions in off-policy optimization while \citet{liu2020provably} additionally use a variational auto-encoder (VAE) to detect out-of-distribution states. We experimented with VAE-filtered pessimistic heuristics in \cref{app:experiments}.
Even pessimistic offline evaluation techniques~\citep{gulcehre2021regularized} can be useful in \algo, since function approximation often induces extrapolation errors~\cite{lu2018non}. 

\vspace{-2mm}
\paragraph{Heuristic pessimism vs. admissibility.} 
\rev{Our concept of heuristic pessimism can be easily confused for the well-established notion of \emph{admissibility}~\citep{russell2020}, but in fact they are opposites. Namely, an admissible heuristic never \emph{underestimates} $V^*$ (in the return-maximization setting), 
while a pessimistic one never \emph{overestimates}  $V^*$. Similarly, our notion of improvability is distinct from \emph{consistency}: they express related ideas, but with regards to pessimistic and admissible value functions, respectively. Thus, counter-intuitively from the planning perspective, our work shows that for policy \emph{learning}, \emph{in}admissible heuristics are desirable. \citet{pearl1981} is one of the few works that has analyzed desirable implications of heuristic inadmissibility in planning.
}

\vspace{-2mm}
\paragraph{Other warm-starting techniques.}
\algo is a new way to warm-start online RL methods.
\citet{bianchi2013heuristically} use a heuristic policy to initialize agents' policies. 
\citet{hester2018deep,vinyals2019grandmaster} train a value function and policy using batch IL and then used them as regularization in  online RL. 
\citet{nair2020accelerating} use off-policy RL on  batch data and fine-tune the learned policy. 
Recent approaches of hybrid IL-RL have strong connections to \algo~\citep{sun2017deeply,cheng2020policy,sun2018truncated}. 
In particular, \citet{cheng2020policy} is a special case of \algo with a max-aggregation heuristic.
%
\citet{farahmand2016truncated} use several related tasks to learn a task-dependent heuristic 
and perform shorter-horizon planning or RL. \rev{Knowledge distillation approaches~\citep{hinton2015distilling} can also be used to warm-start learning, but in contrast to them, \algo\ expects prior knowledge in the form of state value estimates, not features, and doesn't attempt to make the agent internalize this knowledge. A 
\algo\ agent learns from its own environment interactions, using prior knowledge only as guidance. 
Reverse Curriculum approaches~\cite{florensa2017reverse} create short horizon RL problems by initializing the agent close to the goal, and moving it further away as the agent improves. This gradual increase in the horizon inspires the \algo approach. However, \algo does not require the agent to be initialized on expert states and can work with many different base RL algorithms. 
}

\vspace{-1mm}
\section{Discussion and Limitations} \label{sec:disclim}
\vspace{-1mm}

This work is an early step towards theoretically understanding the role and potential of heuristics in guiding RL algorithms. 
We propose a framework, \algo, that can accelerate RL when an informative heuristic is provided. 
\algo induces a horizon-based regularization of RL, complementary to existing warm-starting schemes, and we provide theoretical and empirical analyses to support its effectiveness.
%
While this is a conceptual work without foreseeable societal impacts yet, we hope that it will help counter some of AI's risks by making RL more predictable via incorporating prior into learning.
%

%
We remark nonetheless that the effectiveness of \algo depends on the available heuristic. While \algo can eventually solve the original RL problem even with a non-ideal heuristic, using a bad heuristic can slow down learning. 
Therefore, an important future research direction is to adaptively tune the mixing coefficient based on the heuristic quality with curriculum or meta-learning techniques.
In addition, while our theoretical analysis points out a strong connection between good heuristics for \algo and pessimistic offline RL, techniques for the latter are not yet scalable and robust enough for high-dimensional problems.
Further research on offline RL can unlock the full potential of \algo.

\bibliographystyle{unsrtnat}
\bibliography{references}


\clearpage
\section*{Checklist}

\begin{enumerate}
\item For all authors...
\begin{enumerate}
  \item Do the main claims made in the abstract and introduction accurately reflect the paper's contributions and scope?
    \answerYes{}
  \item Did you describe the limitations of your work?
    \answerYes{\Cref{sec:disclim}.}
  \item Did you discuss any potential negative societal impacts of your work?
    \answerYes{\Cref{sec:disclim}.} It is a conceptual work that doesn't have foreseeable societal impacts yet.
  \item Have you read the ethics review guidelines and ensured that your paper conforms to them?
    \answerYes{}
\end{enumerate}

\item If you are including theoretical results...
\begin{enumerate}
  \item Did you state the full set of assumptions of all theoretical results?
    \answerYes{The assumptions are in the theorem, proposition, and lemma statements throughout the paper.}
	\item Did you include complete proofs of all theoretical results?
    \answerYes{\Cref{app:proofs} and \Cref{app:tech}.}
\end{enumerate}

\item If you ran experiments...
\begin{enumerate}
  \item Did you include the code, data, and instructions needed to reproduce the main experimental results (either in the supplemental material or as a URL)? \answerYes{In the supplemental material.}
  \item Did you specify all the training details (e.g., data splits, hyperparameters, how they were chosen)?
    \answerYes{\Cref{app:experiments}.}
	\item Did you report error bars (e.g., with respect to the random seed after running experiments multiple times)?
    \answerYes{\Cref{sec:exps} and \Cref{app:experiments}.}
	\item Did you include the total amount of compute and the type of resources used (e.g., type of GPUs, internal cluster, or cloud provider)?
     \answerYes{\Cref{app:experiments}.}
\end{enumerate}

\item If you are using existing assets (e.g., code, data, models) or curating/releasing new assets...
\begin{enumerate}
  \item If your work uses existing assets, did you cite the creators?
    \answerYes{References \cite{garage}, \cite{todorov2012mujoco}, and \cite{procgen} in \Cref{ssec:setup}.}
  \item Did you mention the license of the assets?
    \answerYes{In \Cref{ssec:setup}.}
  \item Did you include any new assets either in the supplemental material or as a URL?
    \answerYes{Heuristic computation and scripts to run training.}
  \item Did you discuss whether and how consent was obtained from people whose data you're using/curating?
    \answerNA{}
  \item Did you discuss whether the data you are using/curating contains personally identifiable information or offensive content?
    \answerNA{}
\end{enumerate}

\item If you used crowdsourcing or conducted research with human subjects...
\begin{enumerate}
  \item Did you include the full text of instructions given to participants and screenshots, if applicable?
    \answerNA{}
  \item Did you describe any potential participant risks, with links to Institutional Review Board (IRB) approvals, if applicable?
    \answerNA{}
  \item Did you include the estimated hourly wage paid to participants and the total amount spent on participant compensation?
    \answerNA{}
\end{enumerate}

\end{enumerate}

\clearpage
\appendix

\section{Missing Proofs}  \label{app:proofs}


We provide the complete proofs of the theorems stated in the main paper. 
We defer the proofs of the technical results to \cref{app:tech}.

\subsection{Proof of \cref{th:performance decomposition}}
\PerformanceDecomposition*

First we prove the equality using a new performance difference lemma that we will prove in \cref{app:tech}. This result may be of independent interest.
\begin{restatable}[General Performance Difference Lemma]{lemma}{GPDL} \label{lm:gpdl}
    Consider the reshaped MDP $\widetilde{\MM}$ defined by some $f:\SS\to\R$ and $\lambda\in[0,1]$.
    For any policy $\pi$, any state distribution $d_0$ and any $V:\SS\to\R$, it holds that
    \begin{align*}
        V(d_0) - V^\pi(d_0)  
        &= \frac{\gamma(1-\lambda)}{1-\gamma} \E_{s,a\sim d^\pi}   \E_{s'|s,a} \left[ \h(s')  - V(s')\right] \\
        &\quad
        + \lambda \left( V(d_0) - \widetilde{V}^\pi(d_0) \right)  + \frac{1-\lambda}{1-\gamma}  \left( V(d^{\pi}) - \widetilde{V}^\pi(d^{\pi}) \right)
    \end{align*}
\end{restatable}
Now take $V$ as $\widetilde{V}^*$ in the equality above.
Then we can write 
 \begin{align*}
        V^*(d_0) - V^\pi(d_0) &= \left( V^*(d_0) - \widetilde{V}^*(d_0) \right) + \frac{\gamma(1-\lambda)}{1-\gamma} \E_{s,a\sim d^\pi}   \E_{s'|s,a} \left[ \h(s')  - \widetilde{V}^*(s')\right] \\
        &\quad + \lambda \left( \widetilde{V}^*(d_0) - \widetilde{V}^\pi(d_0) \right)  + \frac{1-\lambda}{1-\gamma}  \left( \widetilde{V}^*(d^{\pi}) - \widetilde{V}^\pi(d^{\pi}) \right)
\end{align*}
which is the regret-bias decomposition.

Next we prove that these two terms are independent of constant offsets. For the regret term, this is obvious because shifting the heuristic by a constant would merely shift the reward by a constant. 
For the bias term, we prove the invariance below.

\begin{restatable}{proposition}{BiasShift}
    $\mathrm{Bias}(\h,\lambda,\pi) = \mathrm{Bias}(\h+b,\lambda,\pi)$ for any $b\in\R$.
\end{restatable}

\begin{proof}
    Notice that any $b\in\R$ and $\pi$,
    $
        \widetilde{V}^\pi(s;f+b) - \widetilde{V}^\pi(s;f) = \sum_{t=0}^\infty (\lambda\gamma)^t (1-\lambda)\gamma b = \frac{(1-\lambda)\gamma }{1-\lambda\gamma}b
    $.
    Therefore, we can derive
    \begin{align*}
        \mathrm{Bias}(\h+b,\lambda,\pi) - \mathrm{Bias}(\h,\lambda,\pi)
        &=  - \Const b  + \frac{\gamma(1-\lambda)}{1-\gamma} \E_{s,a\sim d^\pi}   \E_{s'|s,a} \left[  b -  \Const b \right] \\
        &= \frac{\gamma(1-\lambda)}{1-\gamma} b - \left(  1+  \frac{\gamma(1-\lambda)}{1-\gamma} \right) \Const b
    \end{align*}
    Since
    \begin{align*}
        \left(  1+  \frac{\gamma(1-\lambda)}{1-\gamma} \right) \Const b
        = \frac{1-\gamma + \gamma(1-\lambda)}{1-\gamma}   \Const b
        = \frac{ 1- \gamma\lambda}{1-\gamma}   \Const b
        = \frac{ (1-\lambda)\gamma }{1-\gamma} b
    \end{align*}
    we have $\mathrm{Bias}(\h+b,\lambda,\pi) - \mathrm{Bias}(\h,\lambda,\pi)=0$.
\end{proof}

\subsection{Proof of \cref{th:regret as action gap}}

\RegretAsActionGap*

Define the Bellman backup for the reshaped MDP: 
\begin{align*}
    (\widetilde{\BB} V)(s,a) \coloneqq \widetilde{r}(s,a) + \widetilde{\gamma} \E_{s'|s,a}[V(s')]
\end{align*}
Then by \cref{lm:online value difference} in \cref{app:tech}, we can rewrite the regret as 
    \begin{align*}
        \lambda \left( \widetilde{V}^*(d_0) - \widetilde{V}^\pi(d_0) \right)  + \frac{1-\lambda}{1-\gamma}  \left( \widetilde{V}^*(d^{\pi}) - \widetilde{V}^\pi(d^{\pi}) \right)
        = \E_{\rho^\pi(d_0)} \left[ \sum_{t=0}^\infty \gamma^t \left(  \widetilde{V}^*(s_t)  - (\widetilde{\BB} \widetilde{V}^*)(s_t,a_t) \right) \right]
\end{align*}
Notice the equivalence $ \widetilde{V}^*(s)  - (\widetilde{\BB} \widetilde{V}^*)(s,a) = - \widetilde{A}^*(s,a)$. This concludes the proof.

\subsection{Proof of \cref{th:preserved MDP properties}}

\PreservedProperties*

For the first statement, notice $\widetilde{r}(s,a) \in [0, 1 + \frac{(1-\lambda)\gamma}{1-\gamma}]$. Therefore, we have $\widetilde{V}^\pi(s) \geq 0$ as well as 
\begin{align*}
    \widetilde{V}^\pi(s) 
    &\leq \frac{1}{1-\lambda\gamma} \left( 1 + \frac{(1-\lambda)\gamma}{1-\gamma} \right)\\
    &= \frac{1}{1-\lambda\gamma}  \frac{1-\gamma + (1-\lambda)\gamma}{1-\gamma} = \frac{1}{1-\gamma}
\end{align*}
For the second statement, we just need to show the reshaped reward $\widetilde{r}(s,a)$ is linear in $\phi(s,a)$. This is straightforward because $\E_{s'|s,a}[\h(s')]$ is linear in $\phi(s,a)$.

\subsection{Proof of \cref{th:l-inf bias bound}}
\LinfBiasBound*

By \cref{th:performance decomposition}, we know that 
$\mathrm{Bias}(\h,\lambda,\pi) = \mathrm{Bias}(\h+b,\lambda,\pi)$ for any $b\in\R$. Now consider $b^*\in\R$ such that $\| \h +b^* - V^* \|_\infty \leq \epsilon$. Then by \cref{lm:value change},  we have also $
    \| \h + b^* - \widetilde{V}^{\pi^*} \|_\infty
    \leq \epsilon + \frac{(1-\lambda)\gamma \epsilon}{1-\lambda\gamma}
$.

Therefore, by \cref{th:bias bound}, we can derive with definition of the bias,            
    \begin{align*} 
    \mathrm{Bias}(\h,\lambda,\pi) &= 
    \mathrm{Bias}(\h+b^*,\lambda,\pi) \\
    &\leq
     (1-\lambda)\gamma \left(  \CC(\pi^*,V^*-\h-b^*, \lambda\gamma) + 
     \CC(\pi,\h+b^*-\widetilde{V}^*, \gamma)
     \right)
    \\
    &\leq
     (1-\lambda)\gamma \left(  \CC(\pi^*,V^*-\h-b^*, \lambda\gamma) + 
     \CC(\pi,\h+b^*-\widetilde{V}^{\pi^*}, \gamma)
     \right)
     \\
    &\leq
    (1-\lambda)\gamma
    \left(\frac{\epsilon}{1-\lambda\gamma} +  \frac{1}{1-\gamma}
    (\epsilon + 
   \frac{(1-\lambda)\gamma \epsilon}{1-\lambda\gamma}) \right)\\
   &\leq
    (1-\lambda)\gamma
    \left(\frac{\epsilon}{1-\gamma} +  \frac{1}{1-\gamma}
    (\epsilon + 
   \frac{(1-\lambda)\gamma \epsilon}{1-\gamma}) \right)\\
   &= \frac{2(1-\lambda)\gamma \epsilon}{1-\gamma}
   + \frac{(1-\lambda)^2\gamma^2 \epsilon}{(1-\gamma)^2}
   \leq \frac{(1-\lambda\gamma)^2}{(1-\gamma)^2} \epsilon
    \end{align*}
    
    
\subsection{Proof of \cref{th:bias bound}}
\BiasBound* 
Recall the definition of bias: 
\begin{align*}
 \mathrm{Bias}(\h,\lambda,\pi) =
        \left( V^*(d_0) - \widetilde{V}^*(d_0) \right) + \frac{\gamma(1-\lambda)}{1-\gamma} \E_{s,a\sim d^\pi}   \E_{s'|s,a} \left[ \h(s')  - \widetilde{V}^*(s')\right]  
\end{align*}
For the first term, we can derive by performance difference lemma (\cref{lm:pdl}) and \cref{lm:value difference}
\begin{align*}
        V^*(d_0) - \widetilde{V}^*(d_0)  &\leq V^*(d_0) - \widetilde{V}^{\pi^*}(d_0) \\
        &= (1-\lambda)\gamma  \E_{\rho^{\pi^*}(d_0)} \left[ \sum_{t=1}^\infty (\lambda\gamma)^{t-1}  ( V^*(s_t) -\h(s_t)) \right] 
        = (1-\lambda)\gamma  \CC(\pi, V^*-f, \lambda\gamma)
\end{align*}
For the second term, we can rewrite it as 
\begin{align*}
    \frac{\gamma(1-\lambda)}{1-\gamma} \E_{s,a\sim d^\pi}   \E_{s'|s,a} \left[ \h(s')  - \widetilde{V}^*(s')\right] 
    &=   \gamma(1-\lambda)  \E_{\rho^{\pi}(d_0)} \left[ \sum_{t=1}^\infty \gamma^{t-1}  ( \h(s_t) - \widetilde{V}^*(s_t) ) \right] \\
    &=  (1-\lambda)\gamma  \CC(\pi^*, f-\widetilde{V}^*, \gamma)
\end{align*}

\subsection{Proof of \cref{th:improvable heuristic}}

\ImprovableHeuristic*

Let $d_t^\pi(s;s_0)$  denote the state distribution at the $t$th step after running $\pi$ starting from $s_0\in\SS$ in $\MM$ (i.e. $d_0^\pi(s;s_0) = \one\{s=s_0\}$). Define the mixture
\begin{align} \label{eq:average state distribution in reshaped mdp}
\widetilde{d}_{s_0}^{\pi}(s) \coloneqq (1-\widetilde{\gamma})\sum_{t=0}^\infty \widetilde{\gamma}^t d_t^\pi(s ; s_0)
\end{align}
where we recall the new discount $\widetilde{\gamma}=\gamma\lambda$
By performance difference lemma (\cref{lm:pdl}), we can write for any policy $\pi$ and any $s_0 \in \smash{}$
\begin{align*}
    \widetilde{V}^\pi(s_0) - \h(s_0) = \frac{1}{1-\lambda\gamma} \E_{\widetilde{d}_{s_0}^{\pi}}[ (\widetilde{\BB} \h)(s,a) - \h(s)]
\end{align*}
Notice that
\begin{align*}
    (\widetilde{\BB} \h)(s,a) 
    &=\widetilde{r}(s,a)  + \widetilde{\gamma} \E_{s'|s,a} [\h(s')] \\
    &= r(s,a) + (1-\lambda) \gamma \E_{s'|s,a}[\h(s')] + \lambda\gamma \E_{s'|s,a} [\h(s')] \\
    &= r(s,a) + \gamma \E_{s'|s,a}[\h(s')] = (\BB\h)(s,a)
\end{align*}
Let $\pi$ denote the greedy policy of $\argmax_a (\BB \h) (s,a)$.
Then we have, by the improvability assumption we have $(\BB\h)(s,\pi) - \h(s)\geq 0$ and therefore, 
\begin{align*}
    \widetilde{V}^*(s_0) \geq \widetilde{V}^\pi(s_0) 
    &= \h(s_0) + \frac{1}{1-\lambda\gamma} \E_{\widetilde{d}_{s_0}^{\pi}}[ (\widetilde{\BB} \h)(s,a) - \h(s)] \\
    &= \h(s_0) + \frac{1}{1-\lambda\gamma} \E_{\widetilde{d}_{s_0}^{\pi}}[ (\BB\h)(s,a) - \h(s)] \\
    &\geq \h(s_0)
\end{align*}
Since $s_0$ is arbitrary above, we have the desired statement.

\subsection{Proof of \cref{th:bellman pessimism and improvable heuristic}}

\BellmanPessimismAndImprovableHeuristic*

The proof  is straightforward: 
We have $\max_a (\BB\h)(s,a) \geq (\BB\h)(s,\pi) \geq Q(s,\pi) = \h(s)$, which is the definition of $\h$ being improvable. For the argument of uniform lower bound, we chain the assumption $Q(s,a)\leq (\BB \h)(s,a)$: 
\begin{align*}
    \h(s) = Q(s,\pi') &= 
     r(s,\pi') + \gamma \E_{s'|s,\pi'}[\h(s')] \\
    &\leq r(s,\pi') + \gamma \left(r(s',\pi'),  + \gamma \E_{s''|s',\pi'}[\h(s'')] \right) \\
    &\leq V^{\pi'}(s)
\end{align*}

\section{Technical Lemmas} \label{app:tech}

    \subsection{Lemmas of Performance Difference}

    Here we prove a general performance difference for the $\lambda$-weighting used in the reshaped MDPs.
    \GPDL*
    Our new lemma includes the two below performance difference lemmas in the literature as special cases. 
    \cref{pr:lambda weighted pdf} can be obtained by setting $V=f$;  \cref{lm:pdl} can be obtained by further setting $\lambda=0$ (that is, \cref{lm:pdl} is a special case of \cref{pr:lambda weighted pdf} with $\lambda=0$; and \cref{lm:gpdl} generalizes both). The proofs of these existing performance difference lemmas do not depend on the new generalization in \cref{lm:gpdl}, please refer to~\citep{kakade2002approximately,cheng2020policy} for details. 
    \begin{lemma}[Performance Difference Lemma~\citep{kakade2002approximately,cheng2020policy} ]\label{lm:pdl}
        For any policy $\pi$, any state distribution $d_0$ and any $V:S\to\R$, it holds that 
        \begin{align*}
            V(d_0) - V^\pi(d_0) = \frac{1}{1-\gamma} \E_{d^\pi}[V(s) - (\BB V) (s,a)]
        \end{align*}
    \end{lemma}
	\begin{restatable}[$\lambda$-weighted Performance Difference Lemma~\citep{cheng2020policy}]{lemma}{LambdaWeightedPDL} \label{pr:lambda weighted pdf}
		For any policy $\pi$, $\lambda\in[0,1]$, and $f:\SS\to\R$, it holds that 
       \begin{align*}
        f(d_0) - V^\pi(d_0)  
        &= \lambda \left( f(d_0) - \widetilde{V}^\pi(d_0) \right)  + \frac{1-\lambda}{1-\gamma}  \left( f(d^{\pi}) - \widetilde{V}^\pi(d^{\pi}) \right)
    \end{align*}
	\end{restatable}

    \subsubsection{Proof of \cref{lm:gpdl}}
        First, we use the standard performance difference lemma (\cref{lm:pdl}) in the original MDP and \cref{lm:Bellman backup difference} for the first and the last steps below,
        \begin{align*}
            V(d_0) - V^\pi(d_0)
            &= \frac{1}{1-\gamma} \E_{d^{\pi}} \left[ V(s)  - (\BB V)(s,a)  \right]\\
            &= \frac{1}{1-\gamma} \E_{d^{\pi}} \left[ (\widetilde{\BB} V)(s,a) - (\BB V) (s,a)  \right] + \frac{1}{1-\gamma} \E_{d^{\pi}} \left[ V(s)  - (\widetilde{\BB} V)(s,a)  \right]  \\
            &= \frac{\gamma(1-\lambda)}{1-\gamma} \E_{s,a\sim d^\pi}   \E_{s'|s,a} \left[ \h(s')  - V(s')\right]
            + \frac{1}{1-\gamma} \E_{s,a\sim d^{\pi}} \left[ V(s)  - (\widetilde{\BB} V)(s,a)  \right]  
        \end{align*}
        
       Finally, substituting the equality in \cref{lm:online value difference} into the above equality concludes the proof.

\subsection{Properties of reshaped MDP}

The first lemma is the difference of Bellman backups.
\begin{lemma} \label{lm:Bellman backup difference}
    For any  $V:\SS\to\R$,
    \begin{align*}
        (\BB V)(s,a) - (\widetilde{\BB} V) (s,a)
        &= (1-\lambda) \gamma \E_{s'|s,a}[ V(s') -\h(s')]
    \end{align*}
\end{lemma}
\begin{proof}
    The proof follows from the definition of the reshaped MDP:
    \begin{align*}
        &(\BB V)(s,a) - (\widetilde{\BB} V) (s,a)\\
        &= r(s,a) + \gamma \E_{s'|s,a}[ V(s') ] - r(s,a) - (1-\lambda)\gamma \mathbb{E}_{s'|s,a}[\h(s')] - \gamma \lambda \E_{s'|s,a}[ V(s') ] \\
        &= (1-\lambda) \gamma \E_{s'|s,a}[ V(s') -\h(s')]
    \end{align*}
\end{proof}

This lemma characterizes, for a policy, the difference in returns. 
\begin{lemma} \label{lm:value difference}
    For any policy $\pi$ and $\h:\SS\to\R$,
    \begin{align*}
        V^\pi(s) - \widetilde{V}^{\pi}(s)
        &=(1-\lambda)\gamma \E_{\rho^{\pi}(s)} \left[ \sum_{t=1}^\infty (\lambda\gamma)^{t-1}  ( V^\pi(s_t) -\h(s_t)) \right]
    \end{align*}
\end{lemma}
\begin{proof}
    The proof is based on performance difference lemma (\cref{lm:pdl}) applied in the reshaped MDP and \cref{lm:Bellman backup difference}.
    Recall the definition $\widetilde{d}_{s_0}^\pi(s)$ in \eqref{eq:average state distribution in reshaped mdp} and define $\widetilde{d}_{s_0}^\pi(s,a) = \widetilde{d}_{s_0}^\pi(s)\pi(a|s)$.
    For any $s_0\in\SS$,
    \begin{align*}
        V^\pi(s_0) - \widetilde{V}^{\pi}(s_0)
        &= \frac{1}{1-\gamma\lambda} \E_{s,a\sim\widetilde{d}_{s_0}^{\pi}} [  V^\pi(s) - \widetilde{\BB} V^\pi (s,a)   ]\\
        &= \frac{1}{1-\gamma\lambda} \E_{s,a\sim\widetilde{d}_{s_0}^{\pi}} [  (\BB V^\pi)(s,a) - (\widetilde{\BB} V^\pi) (s, a) ] \\
        &= \frac{(1-\lambda)\gamma}{1-\gamma\lambda} \E_{s,a\sim\widetilde{d}_{s_0}^{\pi}}  \E_{s'|s,a}[ V^\pi(s') -\h(s')]
    \end{align*}
    Finally, substituting the definition of $\widetilde{d}_{s_0}^{\pi}$ finishes the proof.
\end{proof}

A consequent lemma shows that $ \h$ and $\widetilde{V}^{\pi}$ are close, when  $\h$ and  $V^\pi$ are.
\begin{lemma} \label{lm:value change}
    For a policy $\pi$, suppose $-\epsilon_l \leq \h(s) - V^\pi(s) \leq \epsilon_u$. It holds
    \begin{align*}
        -\epsilon_l - \frac{(1-\lambda)\gamma \epsilon_u}{1-\lambda\gamma} \leq \h(s) - \widetilde{V}^{\pi}(s)
        \leq \epsilon_u + \frac{(1-\lambda)\gamma \epsilon_l}{1-\lambda\gamma}
    \end{align*}
\end{lemma}
\begin{proof}
    We prove the upper bound by \cref{lm:value difference}; the lower bound can be shown by symmetry.
    \begin{align*}
        \h(s) - \widetilde{V}^{\pi}(s)
        &\leq \epsilon_u +  V^\pi(s) - \widetilde{V}^{\pi}(s) \\
        &= \epsilon_u + (1-\lambda)\gamma \E_{\rho^{\pi}(s)} \left[ \sum_{t=1}^\infty (\lambda\gamma)^{t-1}  ( V^\pi(s_t) -\h(s_t)) \right]\\
        &\leq \epsilon_u + \frac{(1-\lambda)\gamma \epsilon_l}{1-\lambda\gamma}
    \end{align*}
\end{proof}

    The next lemma relates online Bellman error to value gaps.
   \begin{lemma} \label{lm:online value difference}
        For any $\pi$ and $V:\SS\to\R$,
        \begin{align*}
            \frac{1}{1-\gamma} \left( \E_{d^{\pi}} \left[ V(s)  - (\widetilde{\BB} V)(s,a)  \right] \right) = \lambda \left( V(d_0) - \widetilde{V}^\pi(d_0) \right)  + \frac{1-\lambda}{1-\gamma}  \left( V(d^{\pi}) - \widetilde{V}^\pi(d^{\pi}) \right)
        \end{align*}
    \end{lemma}
    \begin{proof}
        We use \cref{lm:Bellman backup difference} in the third step below.
        \begin{align*}
        &\E_{d^{\pi}} \left[ V(s)  - (\widetilde{\BB} V)(s,a)  \right]\\
            &=
            \E_{d^{\pi}} \left[ V(s)  - (\widetilde{\BB} \widetilde{V}^\pi)(s,a)  \right]  + \E_{d^{\pi}} \left[ \widetilde{\BB} \widetilde{V}^\pi(s,a)  - (\widetilde{\BB} V)(s,a)  \right] \\
            &= \E_{d^{\pi}} \left[ V(s)  - \widetilde{V}^\pi(s)  \right]  + \E_{d^{\pi}} \left[ (\widetilde{\BB} \widetilde{V}^\pi)(s,a)  - (\widetilde{\BB} V)(s)  \right] \\
            &= \E_{d^{\pi}} \left[ V(s)  - \widetilde{V}^\pi(s)  \right]  - \lambda\gamma \E_{s,a\sim d^{\pi}} \E_{s'|s,a} \left[  \widetilde{V}^\pi(s')  -  V(s')  \right]\\
            &= (1-\gamma) \E_{\rho^\pi(d_0)}\left[  \sum_{t=0}^\infty \gamma^t (V(s_t)  - \widetilde{V}^\pi(s_t)) - \lambda \gamma^{t+1} ( \widetilde{V}^\pi(s_{t+1})  -  V(s_{t+1}) )  \right]\\
            &= (1-\gamma)\lambda( V(d_0)  - \widetilde{V}^\pi(d_0) ) + (1-\gamma)(1-\lambda)  \E_{\rho_\pi(d_0)}\left[ \sum_{t=0}^\infty \gamma^t (V(s_t)  - \widetilde{V}^\pi(s_t)) \right]
        \end{align*}
\end{proof}

\section{Experiments} \label{app:experiments}

\subsection{Details of the MuJoCo Experiments} \label{app:details}

We consider four dense reward MuJoCo environments (Hopper-v2, HalfCheetah-v2, Humanoid-v2, and Swimmer-v2) and a sparse reward version of Reacher-v2.

\rev{
In the sparse reward Reacher-v2, the agent receives a reward of $0$ at the goal state (defined as $\norm{g(s)-e(s)}\leq 0.01$ and $-1$ elsewhere, where $g(s)$ and $e(s)$ denote the goal state and the robot's end-effector positions, respectively. We designed a heuristic $\h(s) = r(s,a) - 100 \norm{e(s)-g(s)} $, as this is a goal reaching task. Here the policy is randomly initialized, as no prior batch data is available before interactions.
}

In the dense reward experiments, we suppose that a batch of data collected by multiple behavioral policies are available before learning, and a heuristic is constructed by an offline policy evaluation algorithm from the batch data; in the experiments, we generated these behavioral policies by running \sac from scratch and saved the intermediate policies generated in training.
We designed this heuristic generation experiment to simulate the typical scenario where offline data collected by multiple policies of various qualities is available before learning. In this case, a common method for inferring what values a good policy could get is to inspect the realized accumulated rewards in the dataset.
For simplicity, we use basic Monte Carlo regression to construct heuristics, where a least squares regression problem was used to fit a fully connected neural network to predict the empirical returns on the trajectories in the sampled batch of data.

\rev{Specifically, for each dense reward Mujoco experiment, we ran SAC for 200 iterations and logged the intermediate policies for every 4 iterations, resulting in a total of 50 behavior policies. In each random seed of the experiment, we performed the following: We used each behavior policy to collect trajectories of at most 10,000 transition tuples, which gave about 500,000 offline data points over these 50 policies. These data were used to construct the Monte-Carlo regression data, which was done by computing the accumulated discounted rewards along sampled trajectories. Then we generated the heuristic used in the experiment by fitting a fully connected NN with (256,256)-hidden layers using default ADAM with step size 0.001 and minibatch size 128 for 30 epochs over this randomly generated dataset of 50 behavior policies.}

For the dense reward Mujoco experiments, we also use behavior cloning (\bc) with $\ell_2$ loss to warm start RL agents based on the same batch dataset of 500,000 offline data points.
The base RL algorithm here is SAC, which is based on the standard implementation of Garage (MIT License)~\cite{garage}.
The policy and the value networks are fully connected neural networks, independent of each other. The policy is Tanh-Gaussian and the value network has a linear head. 

\paragraph{Algorithms.}
We compare the performance of different algorithms below.
\begin{enumerate*}[label=\textit{\arabic*)}]    
\item \bc
\item SAC 
\item SAC with \bc warm start (\sacw)
\item \algo with a zero heuristic and \bc warm start (\zeroalgo)
\item \algo with the Monte-Carlo heuristic and \bc warm start (\mcalgo).
\end{enumerate*}
For the \algo algorithms, the mixing coefficient $\lambda_n$ is scheduled as 
\begin{align*}
    \lambda_n &= \lambda_0 + (1-\lambda_0)\tanh\left( \frac{n-1}{\alpha N-1}  \times \arctan(0.99)\right) / 0.99 \\
    &\eqqcolon \lambda_0 + (1-\lambda_0) c_\omega
\tanh(\omega (n-1) )
\end{align*}
for $n=1,\dots,N$, where $\lambda_0 \in [0,1]$ is the initial $\lambda$ and $\alpha>0$ controls the increasing rate.
This schedule ensures that $\lambda_N=1$ when $\alpha=1$. Increasing $\alpha$ from $1$ makes $\lambda_n$ converge to $1$ slower.

We chose these algorithms to illustrate the effect of each additional warm-start component (\bc and heuristics) added on top of the base algorithm \sac.
\zeroalgo is \sacw but with an extra $\lambda$ schedule described above that further lowers the discount, whereas \sac and \sacw keep a constant discount factor.

\paragraph{Evaluation and Hyperparameters.}
In each iteration, the RL agent has a fixed sample budget for environment interactions, and its performance is measured in terms of the undiscounted accumulated rewards (estimated by 10 rollouts) of the deterministic mean policy extracted from \sac.
The hyperparameters used in the algorithms above were selected as follows. The selection was done by uniformly random grid search\footnote{We ran 300 and 120 randomly chosen configurations from \cref{tb:hp grid search} with different random seeds to tune the base algorithm and the $\lambda$-scheduler, respectively. Then the best configuration was used in the following experiments.} over the range of hyperparameters in \cref{tb:hp grid search} to maximize the AUC of the training curve. 

\begin{table}[ht]
    \begin{center}
        \begin{tabular}{ c|  c }
        Polcy step size  &   [0.00025, 0.0005, 0.001, 0.002] \\
        Value step size  &  [0.00025, 0.0005, 0.001, 0.002] \\
        Target step size & [0.005, 0.01, 0.02, 0.04]  \\
        $\gamma$ & [0.9, 0.99, 0.999] \\
        $\lambda_0$ &[0.90, 0.95, 0.98, 0.99] \\
        $\alpha$ & [$10^{-5}$, 1.0, $10^{5}$] \\
        \end{tabular}
    \end{center}
    \caption{\algo's hyperparameter value grid for the MuJoCo experiments.}
    \label{tb:hp grid search}
\end{table}

First, the learning rates (policy step size, value step size, target step size) and the discount factor of the base RL algorithm, \sac, were tuned for each environment to maximize the performance. This tuned discount factor is used as the de facto discount factor $\gamma$ of the original MDP $\MM$.
Fixing the hyperparameters above, $\lambda_0$ and $\alpha$ for the $\lambda$ schedule of \algo were tuned for each environment and each heuristic.
The tuned hyperparameters and the environment specifications are given in \cref{tb:sparse mujoco exp configs,tb:mujoco exp configs} below. (The other hyperparameters, in addition to the ones tuned above, were selected manually and fixed throughout all the experiments).
\begin{table}[ht]
    \begin{center}
        \begin{tabular}{ c|  c  } 
        Environment & Sparse-Reacher-v2  \\\hline
        Obs. Dim & 11 \\
        Action Dim & 2 \\
        Evaluation horizon & 500 \\
        $\gamma$ & 0.9 \\
        Batch Size & 10000  \\
        Policy NN Architecture & (64,64)  \\
        Value NN Architecture &  (256,256)  \\
        Polcy step size  &  0.00025 \\
        Value step size  & 0.00025\\
        Target step size & 0.02\\
        Minibatch Size & 128 \\
        Num. of Grad. Step per Iter. & 1024  \\ \hline
        \algo $\lambda_0$ & 0.5 \\
        \algo-MC $\alpha$ & $10^5$ \\
        \end{tabular}
    \end{center}
    \caption{Sparse reward MuJoCo experiment configuration details. All the values other than $\lambda$-scheduler's (i.e. those used in \sac) are shared across different algorithms in the comparison. All the neural networks here fully connected and have $\tanh$ activation; the numbers of hidden nodes are documented above.
    Note that when $\alpha=10^5$, effectively $\lambda_n=\lambda_0$ in the training iterations;  when $\alpha=10^{-5}$,  $\lambda_n \approx 1 $ throughout. 
    }
    \label{tb:sparse mujoco exp configs}
\end{table}

\begin{table}[ht]
    \begin{center}
        \begin{tabular}{ c|  c | c|  c|  c  } 
        Environment & Hopper-v2 & HalfCheetah-v2 &  Swimmer-v2 & Humanoid-v2 \\\hline
        Obs. Dim & 11 & 17 & 8 &376\\
        Action Dim & 3 & 6 & 2 & 17\\
        Evaluation horizon & 1000 & 1000 & 1000 & 1000 \\
        $\gamma$ & 0.999 & 0.99 & 0.999 & 0.99 \\
        Batch Size & 4000 & 4000 &  4000 & 10000 \\
        Policy NN Architecture & (64,64) & (64,64) & (64,64) & (256,256)  \\
        Value NN Architecture &  (256,256) & (256,256) & (256,256) &  (256,256) \\
        Polcy step size  &  0.00025 & 0.00025 & 0.0005 & 0.002 \\
        Value step size  & 0.0005 &  0.0005 & 0.0005 & 0.00025\\
        Target step size & 0.02 & 0.04 & 0.0100 & 0.02\\
        Num. of Behavioral Policies &  50 & 50 & 50 & 50 \\
        Minibatch Size & 128 & 128 & 128 & 128 \\
        Num. of Grad. Step per Iter. & 1024 & 1024 & 1024 & 1024 \\ \hline
        \algo-MC $\lambda_0$ & 0.95 & 0.99 & 0.95 & 0.9 \\
        \algo-MC $\alpha$ & $10^5$ & $10^5$ & 1.0 & 1.0 \\
        \algo-zero $\lambda_0$ & 0.98 & 0.99 & 0.99 & 0.95\\
        \algo-zero $\alpha$ & $10^{-5}$ & $10^5$ & 1.0 & $10^{-5}$\\
        \end{tabular}
    \end{center}
    \caption{Dense reward MuJoCo experiment configuration details. All the values other than $\lambda$-scheduler's (i.e. those used in \sac) are shared across different algorithms in the comparison. All the neural networks here fully connected and have $\tanh$ activation; the numbers of hidden nodes are documented above.
    Note that when $\alpha=10^5$, effectively $\lambda_n=\lambda_0$ in the training iterations;  when $\alpha=10^{-5}$,  $\lambda_n \approx 1 $ throughout. 
    }
    \label{tb:mujoco exp configs}
\end{table}

Finally, after all these hyperparameters were decided, we conducted additional testing runs with 30 different random seeds and report their statistics here. The randomness include the data collection process of the behavioral policies, training the heuristics from batch data, BC, and online RL, but the behavioral policies are fixed.

\rev{
While this procedure takes more compute (the computation resources are reported below; tuning the base SAC takes the most compute), 
it produces more reliable results 
without (luckily or unluckily) using some hand-specified hyperparameters or a particular way of aggregating scores when tuning hyperparameters across environments.
Empirically, we also found using constant $\lambda$ around $0.95 \sim 0.98$ leads to 
good performance, though it may not be the best environment-specific choice.
}

\paragraph{Resources.}
Each run of the experiment was done using an Azure Standard\_H8 machine (8 Intel Xeon E5 2667 CPUs; memory 56 GB; base frequency 3.2 GHz; all cores peak frequency 3.3 GHz; single core peak frequency 3.6 GHz). The Hopper-v2, HalfCheetah-v2, Swimmer-v2 experiments took about an hour per run. The Humanoid-v2 experiments took about 4 hours.
No GPU was used.

\paragraph{Extra Experiments with VAE-based Heuristics.}
We conduct additional experiments of \algo using a VAE-filtered  pessimistic heuristic. This heuristic is essentially the same as the Monte-Carlo regression-based heuristic we discussed, except that an extra VAE (variational auto-encoder) is used to classify states into known and unknown states in view of the batch dataset, and then the predicted values of unknown states are set to be the lowest empirical return seen in the dataset. In implementation, this is done by training a state VAE (with a latent dimension of 32) to model the states in the batch data, and then a new state classified as unknown if its VAE loss is higher than 99-th percentile of the VAE losses seen on the batch data. The implementation and hyperparameters are based on the code from \citet{liu2020provably}. We note, however, that this basic VAE-based heuristic does not satisfy the assumption of \cref{th:bellman pessimism and improvable heuristic}. 

These results are shown in \cref{fig:extra results}, where \algo-VAEMC denotes \algo using this VAE-based heuristic. Overall, we see that such a basic pessimistic estimate does not improve the performance from the pure Monte-Carlo version (\mcalgo); while it does improve the results slightly in HalfCheetah-v2, it gets worse results in Humanoid-v2 and Swimmer-v2 compared with \mcalgo. Nonetheless, \algo-VAEMC is still better than the base \sac.

\begin{figure*}[th]
	\centering
    \begin{subfigure}{0.245\textwidth}
		\includegraphics[width=\textwidth]{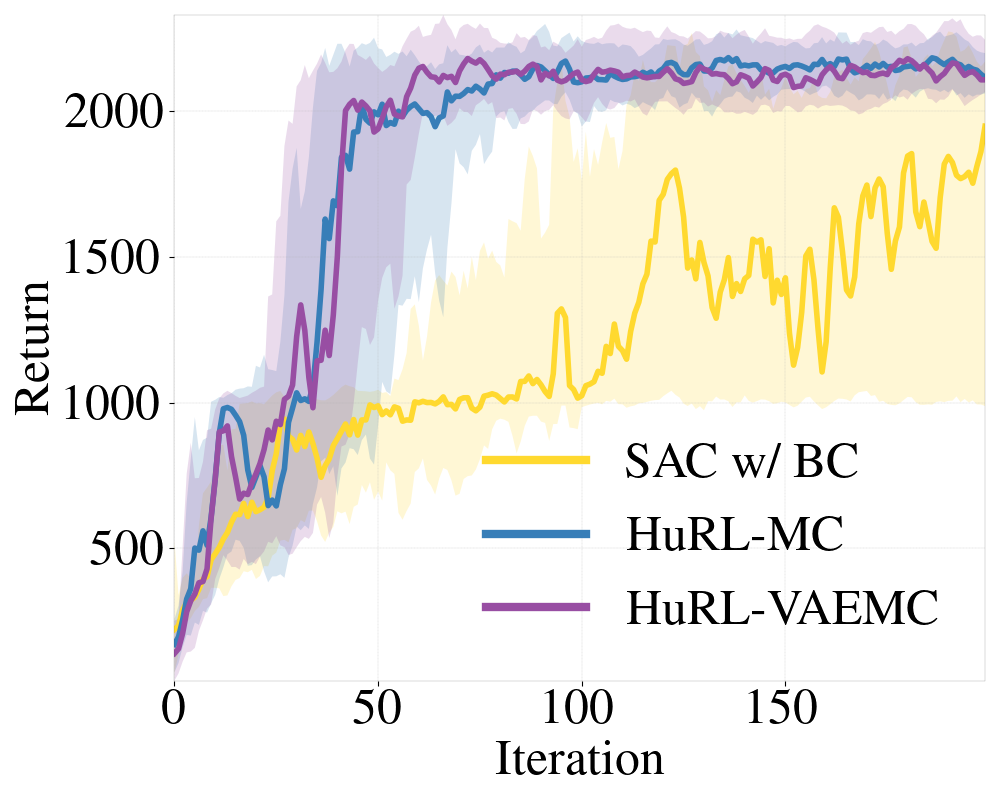}
		\caption{Hopper-v2}
		\label{fig:extra hopper}
	\end{subfigure}
	\begin{subfigure}{0.245\textwidth}
		\includegraphics[width=\textwidth]{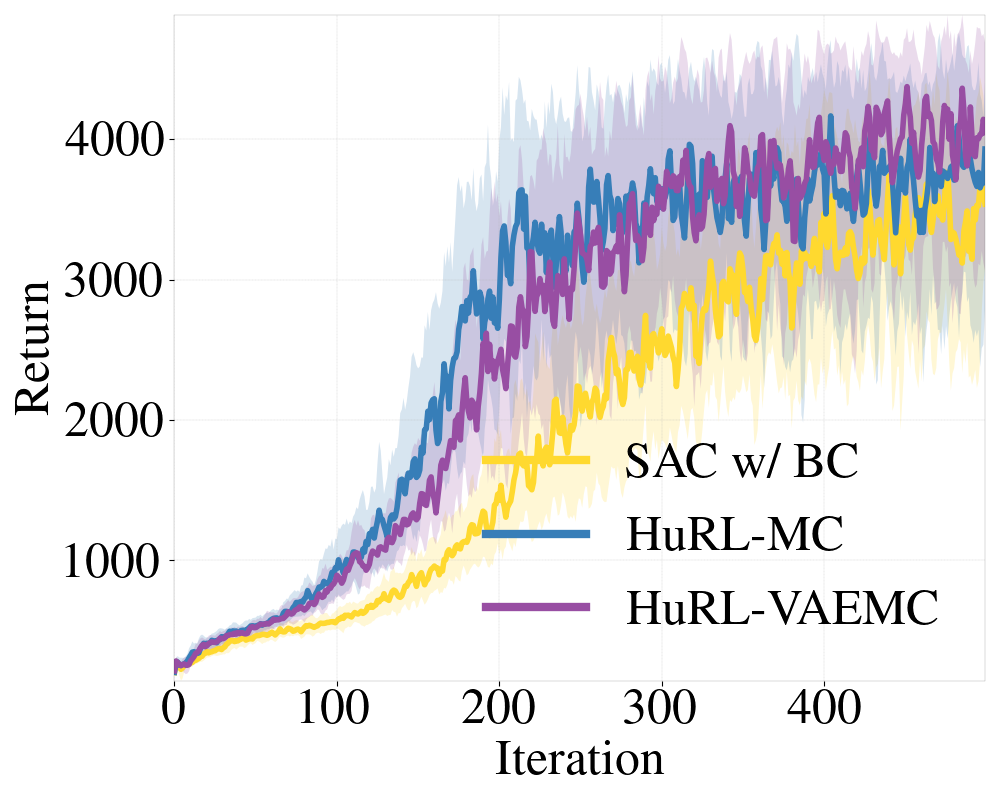}
		\caption{Humanoid-v2}
		\label{fig:extra humanoid}
	\end{subfigure}
	\begin{subfigure}{0.245\textwidth}
		\includegraphics[width=\textwidth]{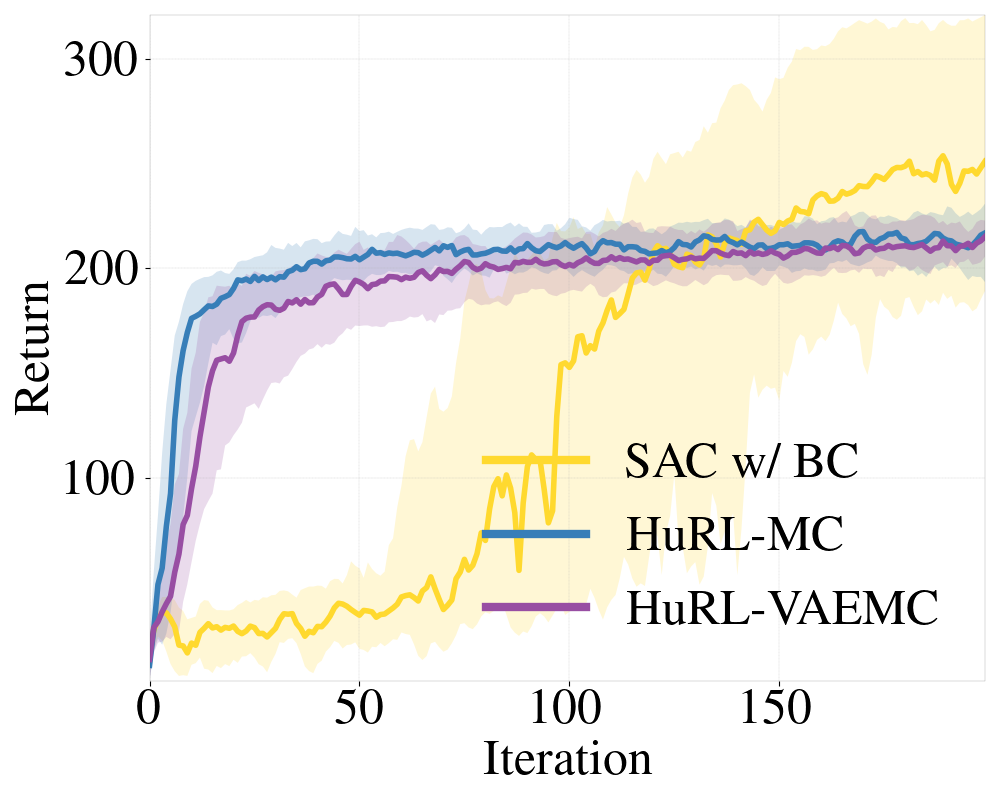}
		\caption{Swimmer-v2}
		\label{fig:extra swimmer}
	\end{subfigure}
	\begin{subfigure}{0.245\textwidth}
		\includegraphics[width=\textwidth]{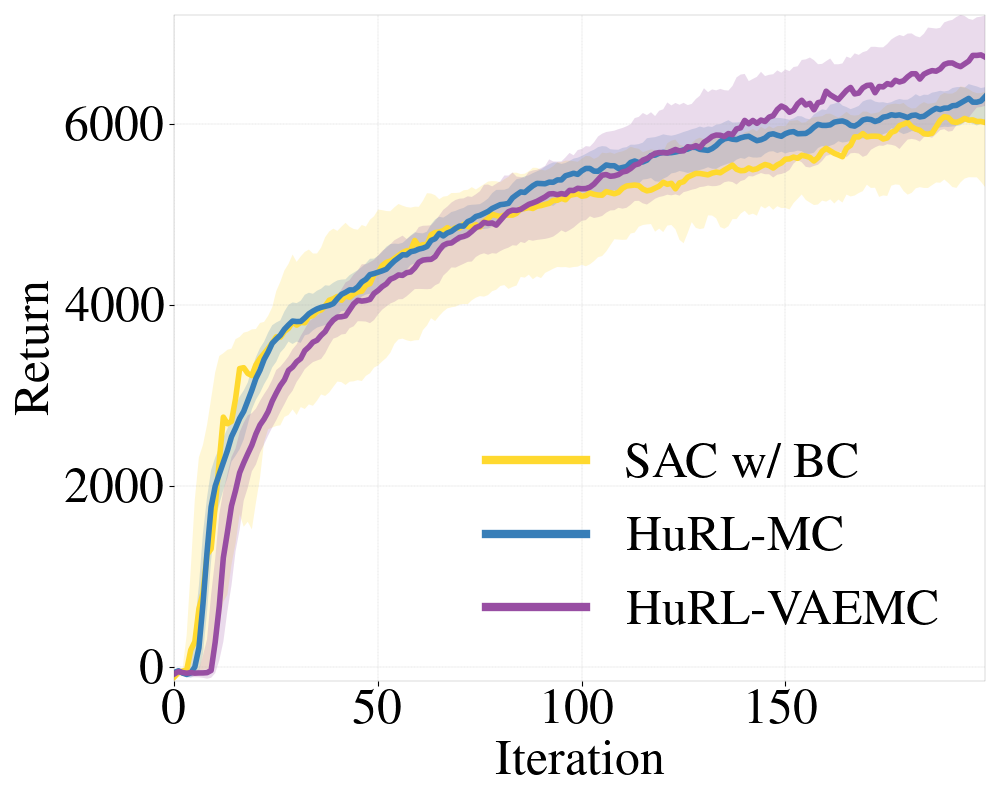}
		\caption{HalfCheetah-v2}
		\label{fig:extra halfcheetah}
	\end{subfigure}

\caption{\small{Extra experimental results of different MuJoCo environments. The plots show the $25$th, $50$th, $75$th percentiles of each algorithm's performance over 30 random seeds.}}
\label{fig:extra results}
\end{figure*}

\subsection{Procgen Experiments}

In addition to MuJoCo environments, where the agent has direct access to the true low-dimensional system state, we conducted experiments on the Procgen benchmark suite \cite{cobbe2020leveraging, procgen}. The Procgen suite consists of 16 procedurally generated Atari-like game environments, whose main conceptual differences from MuJoCo environments are partial observability and much higher dimensionality of agents' observations (RGB images). The 16 games are very distinct structurally, but each game has an unlimited number of levels\footnote{In Procgen, levels aren't ordered by difficulty. They are merely game variations.} that share common characteristics. All levels of a given game are situated in the same underlying state space and have the same transition function but differ in terms of the regions of the state space reachable within each level and in their observation spaces. We focus on the \emph{sample efficiency} Procgen mode \cite{procgen}: in each RL episode the agent faces a new game level, and is expected to eventually learn a single policy that performs well across all levels of the given game.

Besides the differences in environment characteristics between MuJoCo and Procgen, the Procgen experiments are also dissimilar in their design:

\begin{itemize}
    \item In contrast to the MuJoCo experiments, where we assumed to be given a batch of data from which we constructed a heuristic and a warm-start policy, in the Procgen experiments we simulate a scenario where we are given \emph{only} the heuristic function itself. Thus, we don't warm-start the base algorithm with a BC policy when running \algo.
    
    \item In the Procgen experiments, we share a single set of all hyperparameters' values -- those of the base algorithm, those of \algo's $\lambda$-scheduling, and those used for generating heuristics -- across all 16 games. This is meant to simulate a scenario where \algo is applied across a diverse set of problems using good but problem-independent hyperparameters.  
\end{itemize}

\paragraph{Algorithms.}
We used PPO \cite{schulman2017proximal} implemented in RLlib (Apache License 2.0) \cite{pmlr-v80-liang18b} as the base algorithm. We generated a heuristic for each game as follows:
\begin{itemize}
    \item We ran PPO for $8M$ environment interaction steps and saved the policy after every $500K$ steps, for a total of 16 checkpoint policies.
    \item We ran the policies in a random order by executing 12000 environment interaction steps using each policy. For each rollout trajectory, we computed the discounted return for each observation in that trajectory, forming $\langle observation, return \rangle$ training pairs.
    \item We used this data to learn a heuristic via regression. We mixed the data, divided it into batches of 5000 training pairs and took a gradient step w.r.t. MSE computed over each batch. The learning rate was $10^{-4}$.
\end{itemize}

 Our main algorithm, a \algo flavor denoted as PPO-\algo, is identical to the base PPO but uses the Monte-Carlo heuristic computed as above.

%
\paragraph{Hyperparameters and evaluation} The base PPO's hyperparameters in RLlib were chosen to match PPO's performance reported in the original Procgen paper~\cite{cobbe2020leveraging} for the "easy" mode as closely as possible across all 16 games (\citet{cobbe2020leveraging} used a different PPO implementation with a different set of hyperparameters). As in that work, our agent used the IMPALA-CNN$ \times 4$ network architecture~\cite{pmlr-v80-espeholt18a,cobbe2020leveraging} without the LSTM. The heuristics employed the same architecture as well. We used a single set of hyperparameter values, listed in \cref{tb:ppo exp configs}, for all Procgen games, both for policy learning and for generating the checkpoints for computing the heuristics.

\begin{table}[h]
    \begin{center}
\begin{tabular}{c|c}
\hline
Impala layer sizes                         & 16, 32, 32    \\ 
Rollout fragment length                    & 256             \\ 
Number of workers                          & 0 \emph{(in RLlib, this means 1 rollout worker)}  \\ 
Number of environments per worker                  & 64             \\ 
Number of CPUs per worker                  & 5             \\ 
Number of GPUs per worker                  & 0             \\ 
Number of training GPUs                 & 1             \\ 
$\gamma$                                & 0.99          \\ 
SGD minibatch size                             & 2048           \\ 
Train batch size                             & 4000           \\ 
Number of SGD iterations                & 3            \\ 
SGD learning\ rate                          & 0.0005        \\ 
Framestacking                           & off           \\ 
Batch mode                   & truncate\_episodes           \\ 
Value function clip parameter           &10.0            \\ 
Value function loss coefficient           &0.5         \\ 
Value function share layers             &true           \\ 
KL coefficient                    &0.2              \\ 
KL target                    &0.01               \\ 
Entropy coefficient      &0.01          \\ 
Clip parameter                          &0.1           \\
Gradient clip                            &null        \\     
Soft horizon         & False \\        
No done at end: & False \\ 
Normalize actions & False \\     
Simple optimizer & False \\ 
Clip rewards & False \\ 
GAE  $\lambda$                           &0.95      \\ \hline
PPO-\algo\  $\lambda_0$ & 0.99 \\
        PPO-\algo\  $\alpha$ & 0.5 \\
\end{tabular}
\end{center}
    \caption{Procgen experiment configuration details: RLlib PPO's and \algo's hyperparameter values. All the values were shared across all 16 Procgen games.
    }
    \label{tb:ppo exp configs}
\end{table}

\begin{table}[h]
    \begin{center}
        \begin{tabular}{ c|  c }
        $\lambda_0$ &[0.95, 0.97, 0.985, 0.98, 0.99] \\
        $\alpha$ & [0.5, 0.75, 1.0, 3.0, 5.0] \\
        \end{tabular}
    \end{center}
    \caption{\algo's hyperparameter value grid for the Procgen experiments.}
    \label{tb:ppo_hurl}
\end{table}

In order to choose values for PPO-\algo's hyperparameters $\alpha$ and $\lambda_0$, we fixed all of PPO's hyperparameters, took the pre-computed heuristic for each game, and did a grid search over $\alpha$ and $\lambda_0$'s values listed in \cref{tb:ppo_hurl} to maximize the normalized average AUC across all games. To evaluate each hyperparameter value combination, we used 4 runs per game, each run using a random seed and lasting 8M environment interaction steps. The resulting values are listed in \cref{tb:ppo exp configs}. Like PPO's hyperparameters, they were kept fixed for all Procgen environments.

To obtain experimental results, we ran PPO and PPO-\algo\  with the aforementioned hyperparameters on each of 16 games 20 times, each run using a random seed and lasting 8M steps as in \citet{procgen}. We report the 25th, 50th, and 75th-percentile training curves. Each of the reported training curves was computed by smoothing policy performance in terms of unnormalized game scores over the preceding 100 episodes.

\paragraph{Resources.}
Each policy learning run used a single Azure ND6s machine (6 Intel Xeon E5-2690 v4 CPUs with 112 GB memory and base core frequency of 2.6 GHz; 1 P40 GPU with 24 GB memory). A single PPO run took approximately 1.5 hours on average. A single PPO-\algo\  run took approximately 1.75 hours.

\paragraph{Results.} The results are shown in \cref{fig:procgen}. They indicate that, \algo helps despite the highly challenging setup of this experiment: a) environments with a high-dimensional observation space; a) the chosen hyperparameter values being likely suboptimal for individual environments; c) the heuristics naively generated using Monte-Carlo samples from a mixture of policies of wildly varying quality; and d) the lack of policy warm-starting. 
We hypothesize that PPO-\algo's performance can be improved further with environment-specific hyperparameter tuning and a scheme for heuristic-quality-dependent adjustment of \algo's $\lambda$-schedules on the fly.

%
\begin{figure*}[ht!]
 	\centering
    \makebox[\textwidth][c]{
    	\includegraphics[width=1.2\textwidth]{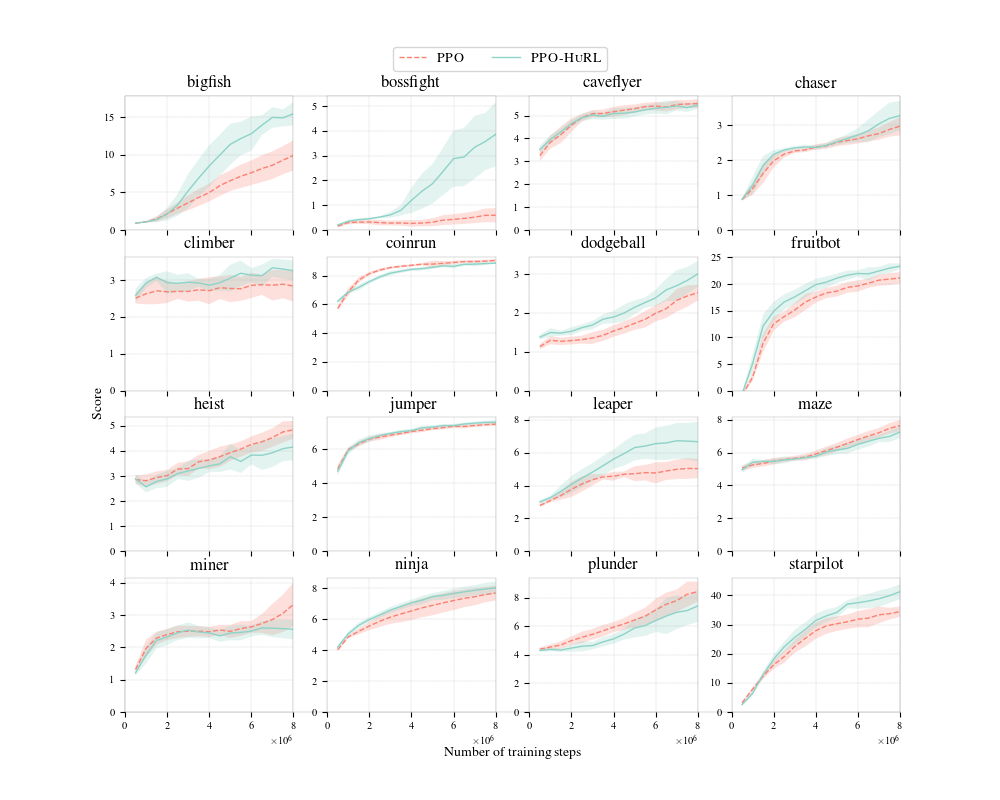}
    }
\caption{\small{PPO-\algo's results on Procgen games. PPO-\algo\ yields gains on half of the games and performs at par with PPO on most of the rest. Thus, on balance, PPO-\algo\ helps despite the highly challenging setup of this experiment, but tuning \algo's  $\lambda$-schedule on the fly depending on the quality of the heuristic can potentially make \algo's performance more robust in settings like this. }}
\label{fig:procgen}
\end{figure*}

\end{document}